\documentclass[letterpaper]{article} 
\usepackage{aaai2026}  
\usepackage{times}  
\usepackage{helvet}  
\usepackage{courier}  
\usepackage[hyphens]{url}  
\usepackage{graphicx} 
\usepackage{natbib}  
\usepackage{caption} 
\usepackage{algorithm}
\usepackage{newfloat}
\usepackage{listings}
\DeclareCaptionStyle{ruled}{labelfont=normalfont,labelsep=colon,strut=off} 
\floatstyle{ruled}
\newfloat{listing}{tb}{lst}{}
\floatname{listing}{Listing}
\usepackage{amsmath,amssymb,amsfonts}%
\usepackage{amsthm}%
\usepackage{mathrsfs}%
\usepackage[title]{appendix}%
\usepackage{xcolor}%
\usepackage{textcomp}%
\usepackage{manyfoot}%
\usepackage{algorithmicx}%
\usepackage{algpseudocode}%
\usepackage{booktabs,multirow,makecell,tabularx}

\theoremstyle{thmstyleone}%
\newtheorem{theorem}{Theorem}%
\newtheorem{corollary}{Corollary}
\newtheorem{proposition}[theorem]{Proposition}%

\theoremstyle{thmstyletwo}%

\theoremstyle{thmstylethree}%
\newtheorem{definition}{Definition}%
\usepackage[capitalize,noabbrev]{cleveref}

\title{Learning from Imperfect Data: Robust Inference of Dynamic Systems using Simulation-based Generative Model}
\author{
    Hyunwoo Cho\textsuperscript{\rm 1}\thanks{First author: H. Cho (chw51@postech.ac.kr)}, 
    Hyeontae Jo\textsuperscript{\rm 2,3}\thanks{Corresponding authors: Email: H. Jo (korea\_htj@korea.ac.kr), H. J. Hwang (hjhwang@postech.ac.kr)}, Hyung Ju Hwang\textsuperscript{\rm 1,4}\footnotemark[2]
}
\affiliations{
    \textsuperscript{\rm 1}Department of Mathematics, Pohang University of Science and Technology, 87 Cheongam-Ro, 37673, Pohang, Republic of Korea\\

    \textsuperscript{\rm 2}Division of Applied Mathematical Sciences, Korea University, 2511 Sejong-ro, 30019, Sejong City, Republic of Korea\\

    \textsuperscript{\rm 3}Biomedical Mathematics Group, Pioneer Research Center for Mathematical and Computational Sciences, Institute for Basic Science, 55 Expo-ro, Yuseong-gu, 34126, Daejeon, Republic of Korea\\

    \textsuperscript{\rm 4}AMSquare Corp., 87 Cheongam-Ro, 37673, Pohang, Republic of Korea\\
}
\usepackage{bibentry}

\begin{document}

\maketitle

\begin{abstract}
System inference for nonlinear dynamic models, represented by ordinary differential equations (ODEs), remains a significant challenge in many fields, particularly when the data are noisy, sparse, or partially observable. In this paper, we propose a Simulation-based Generative Model for Imperfect Data (SiGMoID) that enables precise and robust inference for dynamic systems. The proposed approach integrates two key methods: (1) physics-informed neural networks with hyper-networks that constructs an ODE solver, and (2) Wasserstein generative adversarial networks that estimates ODE parameters by effectively capturing noisy data distributions. We demonstrate that SiGMoID quantifies data noise, estimates system parameters, and infers unobserved system components. Its effectiveness is validated validated through realistic experimental examples, showcasing its broad applicability in various domains, from scientific research to engineered systems, and enabling the discovery of full system dynamics.
\end{abstract}

\begin{links}
    \link{Code}{https://github.com/CHWmath/SiGMoID}
\end{links}

\section{Introduction}

Many scientific fields, such as gene regulation \cite{hirata2002oscillatory, jo2024density}, biological rhythms \cite{forger2024biological}, disease transmission \cite{smith2018influenza, jung2020real, hong2024overcoming}, and ecology \cite{busenberg2012differential}, require the investigation of complex behaviors of $d_y$ system components $\{y_i(t)|i=1,\dots,d_y\}$, over time $t$. The temporal interactions among $y_i(t)$ can be modeled using ordinary differential equations (ODEs) governed by a system function $\mathbf{f}\in\mathbb{R}^{d_y}\times\mathbb{R}^{d_p}\rightarrow\mathbb{R}^{d_y}$, as follows:
\begin{equation}\label{general_ode}
\frac{d\mathbf{y}(t)}{dt}=\mathbf{f}(\mathbf{y}(t),\mathbf{p}),\quad t\in[0,T], 
\end{equation}
where $\mathbf{y}=(y_1,y_2,\dots,y_{d_y})$ denotes a vector comprising the system components and $\mathbf{p}\in\mathbb{R}^{d_p}$ denotes the $d_p$-dimensional vector comprising the model parameters to be estimated based on observed (or experimental) data $\mathbf{y}^{o}(t) = (y^{o}_{1}(t),\dots, y^{o}_{d_y}(t))$ at $N_{o}$ observation time points, $t\in\{t_{1},\dots,t_{N_{o}}\}\subset [0,T]$.

Recent advances in experimental data acquisition have greatly enhanced the ability to monitor dynamic systems, enabling more accurate fitting of solutions to observed data $\mathbf{y}^{o}(t)$. However, these observations are typically collected at discrete time points and are subject to measurement noise. Accordingly, we model the observed data as:
\begin{equation}\label{data_noise}
\mathbf{y}^{o}(t)=\mathbf{y}(t)+\mathbf{e}(t),
\end{equation}
where $\mathbf{e}(t)$ represents the measurement error governed by the noise level, $\sigma$. Additional challenges persist when some elements of the system are difficult to observe due to limitations in measurement resolution \cite{hirata2002oscillatory, smith2018influenza, jo2024density, hong2024overcoming}. In order to address the aforementioned challenges, we classify imperfect datasets, $\mathbf{y}^{o}(t)$, into two distinct types to capture their characteristics effectively:
\begin{itemize}
\item \textbf{Noisy and Sparse (NS):} $i^{th}$ component, $y^{o}_{i}$, is observable, for all $i\in\{1,\dots,d_y\}$, but observations are noisy and recorded at sparse time points (i.e., $\mathbf{y}^{o}=\mathbf{y}+\mathbf{e}$).
\item \textbf{NS with Missing Components (NSMC):} A subset of system components $\{y^{o}_{i}(t),\text{ for some }i\in\{1,\dots,d_y\}\}$ in the NS data is unobservable (See example in \cref{fig1}(a)). We also denote by $S_o$ the set of observable state indices.
\end{itemize}
\begin{figure*}[h!]
\begin{center}
\centerline{\includegraphics[width=2\columnwidth]{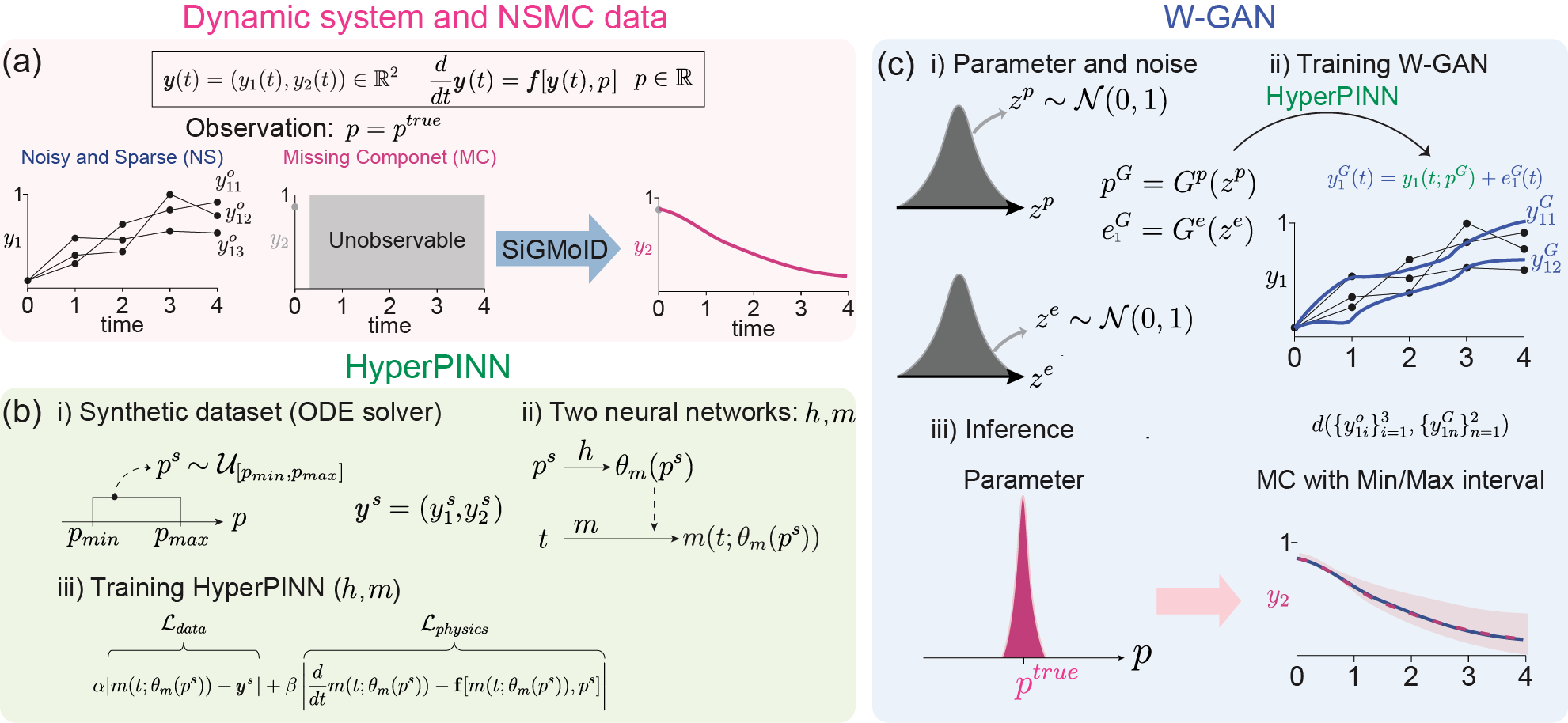}}
\caption{\textbf{Graphical illustration of the functioning of SiGMoID on NSMC data.} (a) ODE system with observable and missing components (b) HyperPINN training from simulated parameter-solution pairs (c) Parameter estimation and missing data recovery using W-GAN}
\label{fig1}
\end{center}
\end{figure*}

To address the challenges associated with NS and NSMC data, we propose a Simulation-based Generative Model for Imperfect Data (SiGMoID). This method utilizes the following two deep-learning models: 1) physics informed neural networks with hyper-networks (HyperPINN), which directly provides solutions of \cref{general_ode} corresponding to any predetermined set of parameters $\mathbf{p}$ (\cref{fig1}(b), and 2) Wasserstein generative adversarial betworks (W-GAN), which adjust the priors by matching the corresponding solutions to imperfect data (\cref{fig1}(c)). By leveraging these two types of models, SiGMoID quantifies the observation noise $\mathbf{e}$ accurately, enabling precise inference of both the underlying system parameters $\mathbf{p}$ and the missing components. 

We also evaluate the efficiency of SiGMoID on four distinct types of dynamic systems: 1) FitzHugh–Nagumo (FN) \cite{fitzhugh1961impulses}, 2) Protein transduction \cite{vyshemirsky2008bayesian}, 3) Gene regulatory networks (\textit{Hes1}) \cite{hirata2002oscillatory}, and 4) Lorenz system \cite{lorenz2017deterministic,stepaniants2024discovering}. Our results reveal the following contributions of SiGMoID:

\textbf{Improved parameter estimation on NS Data:} SiGMoID exhibits superior parameter estimation accuracy compared to existing methods on NS data.

\textbf{Enhanced prediction of full dynamics on NSMC Data:} SiGMoID accurately infers unobserved components on NSMC data, overcoming the limitations of current methods that fail to capture these dynamics completely.

\textbf{Advancing system inference using deep learning-based ODE solvers:} The deep learning-based ODE solver utilized in SiGMoID exhibits potential scalability beyond NS and NSMC datasets, making it applicable to a wide range of datasets.

Therefore, we anticipate that SiGMoID will be broadly applicable across domains ranging from scientific research to engineered systems, facilitating comprehensive analyses of system dynamics and offering robust solutions to challenges involving noisy, sparse, or partially observed data.

\section{Related works}
\subsection{Inference of dynamic systems}
Inference for dynamic systems has been extensively studied in various fields, and numerous methods have been developed to address the challenges posed by NS data. Among these, the penalized likelihood approach \cite{ramsay2007parameter} offers a ODE solution method that bypasses the need for numerical computation. It applies $B$-spline bases to smooth data, while simultaneously incorporating penalties for deviations from the ODE system. Although effective in many scenarios, this method often requires significant manual tuning, particularly for systems with unobserved components, which diminishes their scalability and ease of implementation.

Another line of research employs Gaussian processes (GPs), which provide a flexible and analytically tractable framework for representing system states, derivatives, and observations, thereby enabling parameter inference \cite{calderhead2008accelerating, dondelinger2013ode, wenk2019fast, yang2021inference}. In this context, \cite{dondelinger2013ode} introduced adaptive gradient matching (AGM), which utilizes GP to approximate gradients and fits them to the gradients defined by the ODE system. Further, \cite{wenk2019fast} developed fast GP-based gradient matching (FGPGM) which effectively reduces computational cost by optimizing the GP framework. More recently, \cite{yang2021inference} introduces manifold-constrained Gaussian process inference (MAGI). MAGI explicitly incorporates the ODE structure into the GP model by conditioning the GP’s derivatives on the constraints defined by the ODEs, effectively fitting the nature of the stochastic process to the deterministic dynamics of the system. This approach avoids numerical integration entirely, making it computationally efficient and providing a principled Bayesian framework that ensures consistency between the GP and ODE models. Therefore, to validate the efficiency of the SiGMoID method proposed in this study, we perform a comprehensive comparison with the aforementioned models in terms of their prediction performance on three different examples involving NS and NSMC data.

More recently, deep generative models, such as generative adversarial networks (GANs) \cite{goodfellow2014generative}, have been proposed to imitate data based on dynamic systems. For instance, \cite{kadeethum2021framework, patel2022solution} utilized a GAN to determine the joint probability distribution of parameters and solutions of \cref{general_ode}. Subsequently, \cite{kadeethum2021framework} modified the conditional GAN \cite{mirza2014conditional} to infer parameters based on real data samples. However, GAN-based methods are not explicitly designed to reconstruct unobservable system components, leading to inaccuracies in system inference. Thus, a modification of the GAN architecture and a novel framework that can pre-learn the dynamics systems are required.

\subsection{Simulation-Based Inference of dynamic systems}
Similar to the framework proposed in this study, simulation-based inference (SBI) has been developed to infer the joint probability distributions of solutions and corresponding parameters of \cref{general_ode}, $\pi(\mathbf{y},\mathbf{p})$, by incorporating both deep generative models and simulators based on dynamic systems and their underlying parameters, such as numerical DE solvers  \cite{ramesh2022gatsbi,gloeckler2024all}, with the aim of inferring the true underlying system parameters based on the observed data. For instance, \cite{ramesh2022gatsbi} successfully applied GANs to estimate the distribution of system parameters. More recently, \cite{gloeckler2024all} proposed a model that integrates both transformers and denoising diffusion implicit models to predict posterior distributions on sparse datasets. 

However, these methods have not yet been applied to NSMC datasets. To address this research desideratum, we propose SiGMoID by modifying the SBI strategy. The proposed framework handles NSMC data by combining an ODE solver capable of learning system dynamics with a W-GAN specifically designed to generate noisy data.

\section{Methods}
We develop a Simulation-based Generative Model for Imperfect Data (SiGMoID), which captures system parameters $\mathbf{p}$ in \cref{general_ode} and noise level $\mathbf{e}$ in \cref{data_noise}. The detailed procedure is provided in Methodological Details section of the Appendix, along with pseudo-algorithms (\cref{alg:SiGMoiD}). For better understanding, we present an overview of the key concepts and provide a graphical illustration of SiGMoID on NSMC data in \cref{fig1}, assuming $d_y = 2$, $d_p = 1$, and $N_s = 3$ observed time series in \cref{general_ode} for simplicity. Under this assumption, $\mathbf{y}(t;p^{true})=\mathbf{y}(t) = (y_1(t), y_2(t)) \in \mathbb{R}^2$ denotes the solution of \cref{general_ode} over time $t$ with a true system parameter $p^{true} \in \mathbb{R}$ (\cref{fig1}(a)). Three NS data corresponding to $y_1$, $\{y_{1i}^{o}\}_{i=1}^{3}$, are given based on observations (NS, blue) at time points $t_j=j-1$, for $j=1,2,3,4,5$. However, the data corresponding to $y_{2}$ are not observable (that is, MC). To address this, the SiGMoID framework is utilized to infer the noise distribution of $y_{1}$, estimate the true parameter $p^{true}$, and reconstruct the underlying true solution $y_2$.

In SiGMoID framework, we first set the parameter boundary, $[p_{\text{min}}, p_{\text{max}}]$ based on the empirically determined feasible range of true parameter values (\cref{fig1}(b)-i)). A random parameter value $p^s$ is then sampled from the uniform distribution $\mathcal{U}_{[p_{\text{min}}, p_{\text{max}}]}$. Using a numerical solver, the solution $\mathbf{y}^s(t)=\mathbf{y}(t;p^s)$ of the system considered in (\cref{fig1}(a)) is computed with $p^s$. Subsequently, the pair $(p^s, \mathbf{y}^{s})$ is used to train the HyperPINN framework, which consists of two neural networks---a hypernetwork $h$ and a main network $m$ (\cref{fig1}(b)-ii)). The hypernetwork $h$ takes the sampled parameter $p^s$ as its input and generates weights and biases $\theta_m(p^s)$ for the main network $m$. The main network $m^s(t)=m(t; \theta_m(p^s))$ produce the solution $\mathbf{y}^s(t)$, for an arbitrary time $t\in[0,T]$. For given time collocation points $\{t_{j}^{c}\}_{j=1}^{T_{col}}$, the networks $h$ and $m$ are trained jointly by minimizing the weighted sum of data loss $\mathcal{L}_{data}$ and physics loss $\mathcal{L}_{physics}$ with distinct weights $\alpha,\beta>0$ (\cref{fig1}(b)-iii)):
\begin{align*}
 \mathcal{L}_{data} & = \sum_{j=1}^{T_{col}} \left\Vert m^s(t^c_{j}) - \mathbf{y}(t^c_{j};p^s)\right\Vert^2, \\
 \mathcal{L}_{physics} & = \sum_{j=1}^{T_{col}}  
\Bigg\Vert \frac{d}{dt} m^s(t^c_{j}) - \mathbf{f}\big[m^s(t^c_{j}), p^s \big] \Bigg\Vert^2,\\
 \mathcal{L} &= \alpha\mathcal{L}_{data}+ \beta \mathcal{L}_{physics},
 \end{align*}
where $\|\mathbf{y}\|^2$ denotes $L^2$ norm $\|\mathbf{y}\|^2=y_1^2+\cdots+y_{d_y}^2$.

For W-GAN step (\cref{fig1}(c)), latent variables $z^p$ and $z^e$ are sampled from the standard normal distribution, $\mathcal{N}(0, 1)$ (\cref{fig1}(c)-i)). These serve as inputs to the two generators $G^p$ and $G^e$, to yield a candidate system parameter $p^G$ and data noise $e_{1}^G$ for $\{y_{1i}^{o}\}_{i=1}^{3}$, respectively. Using the HyperPINN trained in (b)-ii), the solution $y_{1}^G=m_{1}^G+e_1^G$ corresponding to $y_{1}$ is then obtained (\cref{fig1}(c)-ii)). By repeating this procedure $N$ times, each time sampling random parameters and noise values (\cref{fig1}(c)-i), we obtain $N$ samples $\{(p_{n}^{G},y_{1n}^{G})\}_{n=1}^{N}$ from the generator. In this work, for simplicity, we take $N=2$. The discrepancy between $\{y_{1i}^o\}_{i=1}^{3}$ and $\{y_{1n}^{G}\}_{n=1}^{2}$ is measured in terms of the Wasserstein distance \begin{gather*} d(\{y_{1i}^o\}_{i=1}^{3},\{y_{1n}^{G}\}_{n=1}^{2}) \\ = 
 \sup_{\phi \in \text{Lip}_{1}}\mathbb{E}_{\mathbf{y}^o}[\phi(\{y_{1i}^o\}_{i=1}^3)]-\mathbb{E}_{\mathbf{y}^G}[\phi(\{y_{1n}^{G}\}_{n=1}^{2})], \end{gather*} where $\{y_{1i}^o\}_{i=1}^3\sim\mathbf{y}^o$ and $\{y_{1n}^{G}\}_{n=1}^{2}\sim\mathbf{y}^G$ denote distributions of observed and generated outputs, respectively. The W-GAN is trained by minimizing this objective. Subsequently, the parameter values $\{p_{n}^G\}_{n=1}^{2}$ generated by $G^p$ are expected to converge within a narrow region containing the true system parameter $p^{true}$ (\cref{fig1}(c)-iii), (left)). Using the parameter distribution, the numerical solver is reapplied to generate MC data for $y_2$ (right). A general description of this procedure is provided in Methodological Details in Appendix, and detailed descriptions of HyperPINN and GAN, including stability measures, hyperparameters, and architecture configurations, are presented in \cref{table:network settings:hyperpinn} to \cref{table:network settings:wgan}.

\section{Experiments}
We apply SiGMoID to four real-world problems, requiring inference of system parameters and reconstruction of MC datasets: 1) the FitzHugh–Nagumo (FN) model \cite{fitzhugh1961impulses}, 2) the protein transduction model \cite{vyshemirsky2008bayesian}, 3) the \textit{Hes1} model \cite{hirata2002oscillatory}, and 4) Lorenz system \cite{lorenz2017deterministic}. For each example, the NS and NSMC datasets are generated using a numerical solver (explicit Runge-Kutta of order four) based on the simulation configurations obtained from the references. In particular, the simulations adopt the parameter configurations and initial conditions established in previous studies. The results of this evaluation demonstrate that SiGMoID not only outperforms existing methods \citep{yang2021inference,stepaniants2024discovering}, but also excels at reconstructing unobserved components. The computational costs for all cases are summarized in \cref{table:runtime} in the Appendix.
\begin{figure}[t!]
\begin{center}
\centerline{\includegraphics[width=\columnwidth]{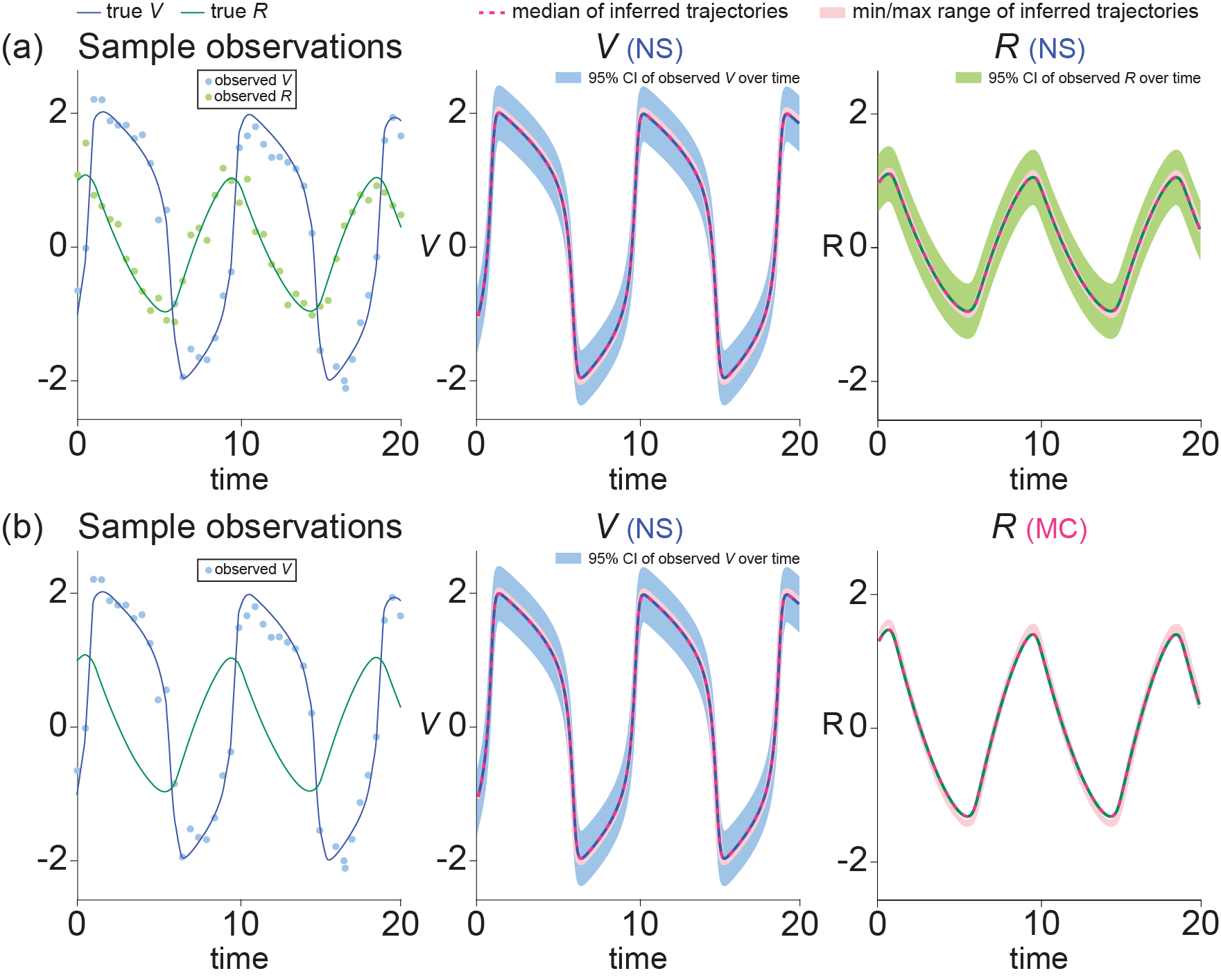}}
\caption{\textbf{System inference for the FN equation.} Demonstration of the capability of SiGMoID to infer true system solutions ($V$ and $R$) for the FN model. (a) NS datasets for $V$ and $R$ are provided (sample observations - blue and green dots). SiGMoID infers the true system solutions ($V$ (NS), $R$ (NS) - red dashed line) accurately. The red region represents the range of the solutions inferred using SiGMoID, while the blue region represents the $95\%$ confidence interval (CI) of the observed component over time. (b) The experimental dataset for $R$ is missing, while the analogue for $V$ is available (Sample observations - blue dots). Despite the absence of $R$ in the dataset, SiGMoID fits both the true $V$ and true $R$ successfully. The CI for the observed $R$ is omitted owing to the absence of any corresponding dataset.}
\label{fig2}
\end{center}
\end{figure}

\subsection{FitzHugh–Nagumo}
Spike potentials in ion channels can be interpreted as the interaction between the membrane potential voltage of a neuron $V$ and the recovery variable associated with the neuron current, $R$. This relationship is described by the FN equations for $\mathbf{y} =(V,R)$. The equations for $\mathbf{y}=(V,R)$ are expressed as follows:
\begin{equation*}f(\mathbf{y},\mathbf{p})=\begin{pmatrix} c\left(V-\frac{V^{3}}{3}+R\right) \\ -\frac{1}{c}\left(V-a+bR\right)\end{pmatrix},
\end{equation*}
where the set of parameters $\mathbf{p}=(a,b,c)$ indicates the equilibrium voltage level for the system $a$, the coupling strength between the recovery variable and the membrane potential $b$, and the timescale and sensitivity of the voltage dynamics $c$.

To generate sample observations, we initialize the true parameters values as $\mathbf{p}^{true}=(0.2,0.2,3)$ and set the initial condition to be $\mathbf{y}(0)=(-1,1)$ \cite{fitzhugh1961impulses}. Using these conditions, we compute the true underlying trajectories (\cref{fig2}(a), Sample observations-true $V, R$). Next, one hundred observed trajectories for $V$ and $R$ are generated corresponding to a noise level of $0.2$ using $41$ observation points to construct the NS dataset (\cref{fig2}(a), Sample observations-observed $V, R$) (See also the detailed scenario reported in \cite{dondelinger2013ode,wenk2019fast}). Under this configuration, the parameter estimation problem becomes non-identifiable (See \cref{fig:FN_CI} in Appendix for details).

SiGMoID is used to infer both the true underlying trajectories (\cref{fig2}(a), $V$ and $R$) and the parameters (\cref{table:params_FN} in Appendix, NS). It is observed to exhibit the lowest root mean square error (RMSE), $\frac{1}{N_s}\sum_{j=1}^{N_{s}}\sqrt{\frac{1}{N_t}\sum_{k=1}^{N_{o}}|y_{true,i}(t_{k})-y_{i}^{G}(t_{k})|^2}$, between the true $y_{true,i}(t)=y_{i}(t;p^{true})$ and estimated trajectories $y_{ij}^{G}(t) = y_{i}(t;p_{j}^{G})$ for each component $i$ compared to the other three methods---MAGI \cite{yang2021inference}, FGPGM \cite{wenk2019fast}, and AGM \cite{dondelinger2013ode} (\cref{table:RMSE_FN}, NS). Further, the parameters inferred using SiGMoID are observed to be the closest to the true values compared to those obtained using the other methods. These results demonstrate that SiGMoID not only accurately reconstructs the true underlying trajectories but also outperforms existing methods (MAGI, FGPGM, and AGM) in terms of estimation accuracy, establishing its effectiveness and reliability.
\begin{table}[t!]
\vskip 0.1in
\begin{center}
\begin{small}
\begin{sc}
\begin{tabularx}{0.47\textwidth}{l *{3}{>{\centering\arraybackslash}X}}
\toprule
Data & Method & $V$ & $R$ \\
\midrule
\multirow{4.2}{*}{NS} & SiGMoID    & $\mathbf{0.028}$ & $\mathbf{0.012}$ \\
& MAGI & $0.103$ & $0.070$ \\
& FGPGM & $0.257$ & $0.094$ \\
& AGM & $1.177$ & $0.662$  \\
\midrule
NSMC & SiGMoID    & $0.038$ & $0.016$  \\
\bottomrule
\end{tabularx}
\end{sc}
\end{small}
\end{center}
\caption{Trajectory RMSEs of each component in the FN system, comparing the average trajectory RMSE values of the four methods over one hundred simulated datasets.}
\label{table:RMSE_FN}
\end{table}
Although the neuronal potential voltage $V$ is easily observable, $R$ cannot be directly measured experimentally and must be inferred based on the dynamic system \cite{fitzhugh1961impulses, samson2025inference}. To this end, we apply SiGMoID to the NSMC dataset by removing the sample observations for $R$ in \cref{fig2}(a) (\cref{fig2}(b), sample observations). Despite the absence of observed data for $R$, the parameters inferred using SiGMoID based on the NSMC dataset are observed to be accurate and closely aligned with the true values (\cref{table:params_FN} in Appendix, NSMC).

\subsection{Protein transduction}
\begin{figure*}[t!]
\vskip 0.1in
\begin{center}
\centerline{\includegraphics[width=2\columnwidth]{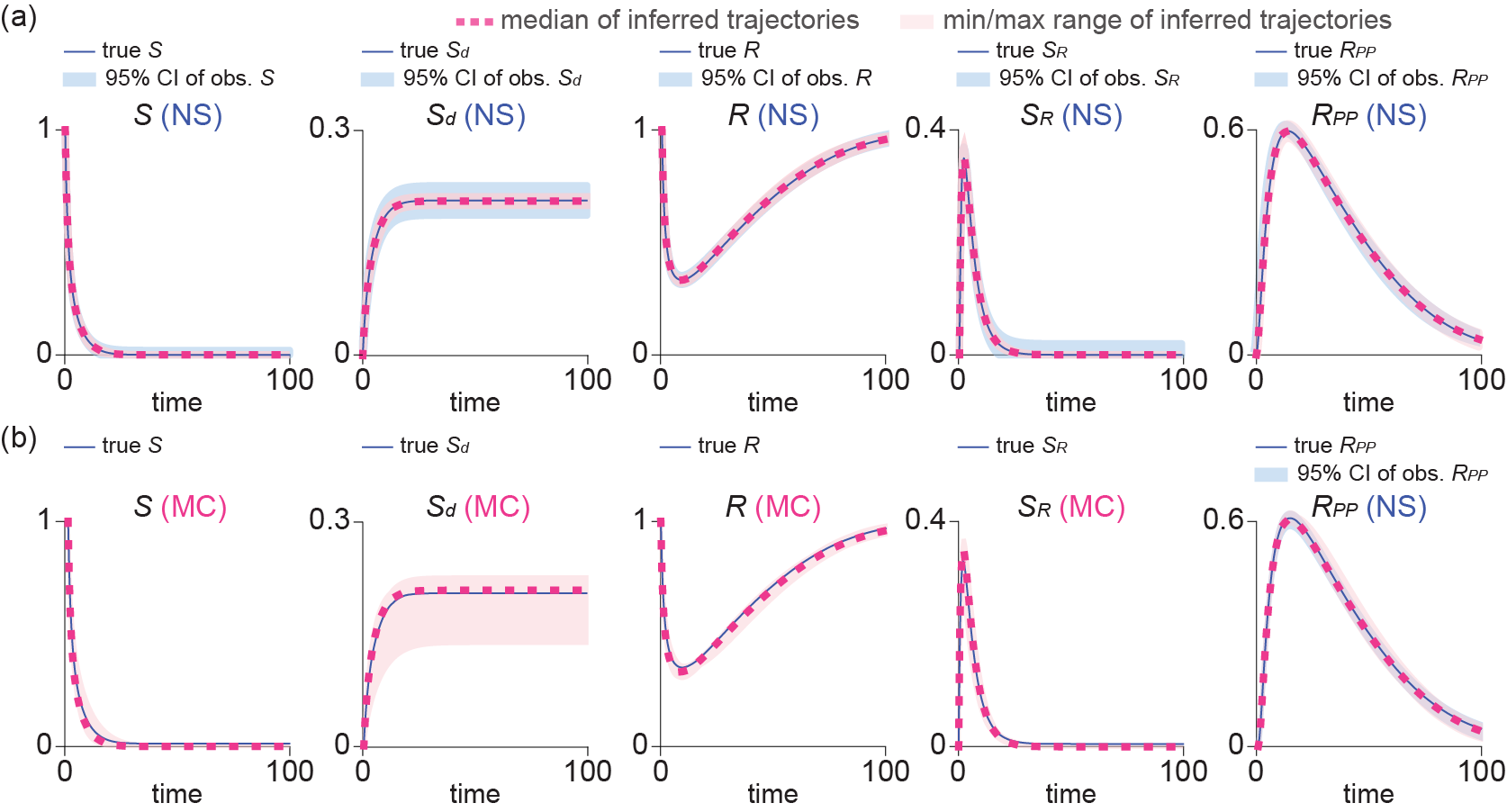}}
\caption{\textbf{System inference for a protein transduction model.} We infer the true system solutions (true $S$, $S_{d}$, $R$, $S_{R}$, and $R_{pp}$) using SiGMoID. (a) When all the components are given in the NS dataset, the solutions inferred using SiGMoID match the true solutions accurately. (b) Similar to (a), the solutions inferred using SiGMoID match the true solutions accurately, even though the components, $S$, $S_{d}$, $R$, and $S_{R}$ are not observable (MC).}
\label{fig3}
\end{center}
\end{figure*}

The dynamics of signaling proteins and their interactions in biochemical pathways can be described using a reaction network with variables $\mathbf{y} = (S, S_{d}, R, S_{R}, R_{pp})$. Specifically, the signaling protein $S$ interacts with the receptor protein $R$ to form the complex $S_R$ with a binding rate of $k_2$. The complex $S_R$ can dissociate back into $S$ and $R$ with a dissociation rate of $k_3$. In addition, $S_R$ facilitates the activation of $R$ into its phosphorylated form $R_{pp}$ with an activation rate of $k_4$. The signaling protein $S$ also degrades into its inactive form $S_d$ with a degradation rate of $k_1$. Finally, the activated receptor $R_{pp}$ is deactivated via Michaelis-Menten kinetics, where $V$ represents the maximum deactivation rate and $K_m$ denotes the Michaelis constant, which defines the concentration of $R_{pp}$ at which the deactivation rate becomes half of its maximum value. These dynamics are described by the following system of ODEs with $\mathbf{p}=(k_{1}, k_{2}, k_{3}, k_{4},V,K_{m})$:
\begin{equation*} f(\mathbf{y},\mathbf{p})=\begin{pmatrix} -k_{1}S - k_{2}SR + k_{3}S_{R}\\ k_{1}S \\ -k_{2}SR - k_{3}S_{R} + \frac{V R_{pp}}{K_{m}+R_{pp}}\\ k_{2}SR - k_{3}S_{R} - k_{4}S_{R} \\ k_{4}S_{R} - \frac{V R_{pp}}{K_{m}+R_{pp}} \end{pmatrix}. \end{equation*} 

To generate sample observations, we set $\mathbf{p}^{true}=(0.07, 0.6, 0.05, 0.3, 0.017, 0.3)$ and $\mathbf{y}(0)=(1, 0, 1, 0, 0)$. Using these conditions, the true underlying trajectories are computed (\cref{fig3}(a), Sample observations-true $S, S_{d} ,R, S_{R},R_{pp}$). Subsequently, one hundred observed trajectories are generated corresponding to a noise level of $0.01$, and $26$ observation time points are used to construct the NS dataset (\cref{fig3}(a), Sample observations- $95\%$ CI of obs. $S, S_{d}, R, S_{R}, R_{pp}$). Notably, the selected noise level represents a severely challenging scenario, as described in \cite{vyshemirsky2008bayesian}.

Subsequently, both the true underlying trajectories (\cref{fig3}(b), $S, S_{d} ,R, S_{R}$ and $R_{pp}$) and the parameters are inferred using SiGMoID. As demonstrated previously, SiGMoID achieves the lowest RMSE values between the true and estimated trajectories compared to the other three methods (MAGI, FGPGM, and AGM) (\cref{table:RMSE_pt}, NS). Further, identifiability problems are observed in the parameter estimates for the protein transduction model \cite{dondelinger2013ode, wenk2019fast}, where the system exhibits low sensitivity to certain parameters, complicating the unique determination of their values based on the observed data. Despite this phenomenon, the parameter estimates obtained using SiGMoID are observed to match the true values most accurately (\cref{table:params_pt} in Appendix, NS).
\begin{table}[t!]
\begin{center}
\begin{small}
\begin{sc}
\begin{tabularx}{0.5\textwidth}{l *{6}{>{\centering\arraybackslash}X}}
\toprule
Data & Method & $S$ & $S_d$ & $R$ & $S_{R}$ & $R_{pp}$ \\
\midrule
\multirow{4.2}{*}{NS} & SiGMoID    & $\mathbf{0.3}$ & $\mathbf{0.8}$ & $\mathbf{1.2}$ & $\mathbf{0.4}$ & $\mathbf{1.3}$\\
& MAGI & $12.2$ & $4.3$ & $16.7$ & $13.5$ & $13.6$ \\
& FGPGM    & $12.8$ & $8.9$ & $21.0$ & $13.6$ & $30.9$ \\
& AGM    & $67.1$ & $312.5$ & $413.8$ & $98.0$ & $297.3$ \\
\midrule
NSMC & SiGMoID    & $11.2$ & $11.2$ & $4.9$ & $3.8$ & $2.9$ \\
\bottomrule
\end{tabularx}
\end{sc}
\end{small}
\end{center}
\caption{Trajectory RMSEs of each component in the protein transduction system, comparing the average trajectory RMSE of the four methods over one hundred simulated datasets. (scaled by $\times 10^{3}$)}
\label{table:RMSE_pt}
\end{table}

Similar to the application of SiGMoID to the FN model, we obtain the NSMC dataset by removing sample observations corresponding to all components except $R_{pp}$ (NS) (\cref{fig3}(b)), which is consistent with the experimental results reported in \cite{vyshemirsky2008bayesian}. Notably, this unobserved component renders the parameter inference problem non-identifiable, (See \cref{fig:Protein_CI} in Appendix for details). Despite the absence of all components except $R_{pp}$, SiGMoID exhibits robust performance in inferring both trajectories (\cref{fig3}(b), $S, S_{d} ,R, S_{R},R_{pp}$ and \cref{table:RMSE_pt} NSMC) and parameters (\cref{table:params_pt} in Appendix, NSMC).

\subsection{\textit{Hes1} model}
\begin{figure}[ht!]
\begin{center}
\centerline{\includegraphics[width=\columnwidth]{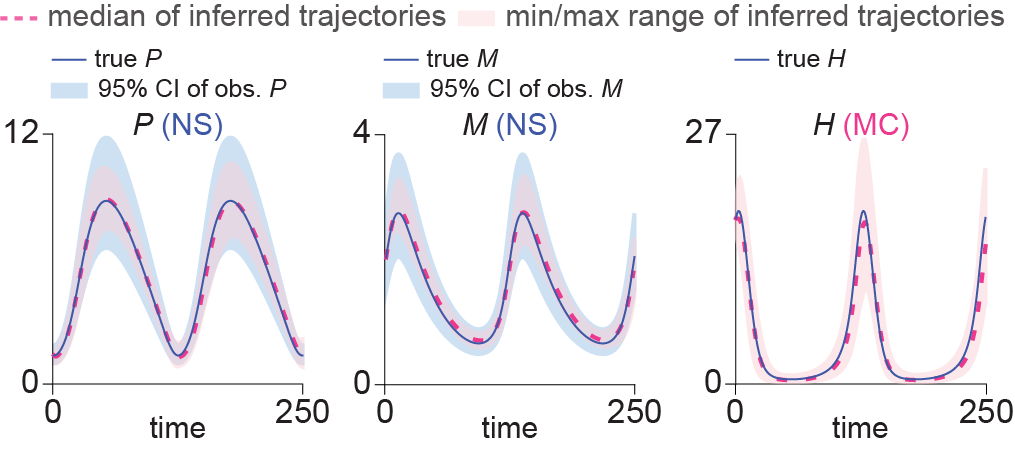}}
\caption{\textbf{System inference for a \textit{Hes1} model.} Datasets for the two components, $P$ and $M$, in the \textit{Hes1} model are provided as NS data, while the component, $H$, is unobserved (MC). SiGMoID accurately infers not only the true system solutions for all components but also captures the noise in the data ($95\%$ CI of obs. $P$ and $M$).}
\label{fig4}
\end{center}
\end{figure}

\begin{table}[h]
\begin{center}
\begin{small}
\begin{sc}
\begin{tabularx}{0.47\textwidth}{l *{3}{>{\centering\arraybackslash}X}}
\toprule
Method & $P$ & $M$ & $H$ \\
\midrule
SiGMoID    & $\mathbf{0.38}$ & $\mathbf{0.12}$ & $\mathbf{1.96}$ \\
MAGI & 0.97 & 0.21 & 2.57 \\
\cite{ramsay2007parameter} & 1.30 & 0.40 & 59.47 \\
\bottomrule
\end{tabularx}
\end{sc}
\end{small}
\end{center}
\caption{Trajectory RMSEs of each component in the hes1 system, comparing the average trajectory RMSE of the three methods over $2,000$ simulated datasets. Note that $H$ denotes the missing component.}
\label{table:RMSE_hes1}
\end{table}

The oscillation of \textit{Hes1} mRNA $M$ and \textit{Hes1} protein $P$ levels in cultured cells can be described using a dynamic system \cite{hirata2002oscillatory}. This system postulates the involvement of an interacting factor $H$ that contributes to stable oscillations, representing a biological rhythm \cite{forger2024biological}. The system of ODEs for the three-component state vector $\mathbf{y} = (P, M, H)$ is expressed as follows:
\begin{equation}
f(\mathbf{y},\mathbf{p})=\begin{pmatrix}
-aPH+bM-cP \\
-dM+\frac{e}{1+P^2} \\
-aPH+\frac{f}{1+P^2}-gH
\end{pmatrix},\label{eqn_hes1}
\end{equation}
where $\mathbf{p}=(a,b,c,d,e,f,g)$ are the associated parameters.

The \textit{Hes1} model is a prime example of the challenges of inference in systems with unobserved components and asynchronous observation times. In \cite{hirata2002oscillatory}, the theoretical oscillatory behavior of the system was established using the true parameter values $a = 0.022, b = 0.3, c = 0.031, d = 0.028, e = 0.5, f = 20$, and $g = 0.3$, revealing an oscillation cycle of approximately $2$ hours. The initial conditions were defined as the lowest value of $P$ during oscillatory equilibrium, with $P = 1.439, M = 2.037$, and $H = 17.904$. To construct the NSMC dataset from the true solution in this study, a noise level of $\sigma = 0.15$ is adopted based on \cite{yang2021inference}. This choice is aligned with the reported standard errors, which account for approximately $15\%$ of $P$ (protein) and $M$ (mRNA) levels in repeated measurements. Accordingly, simulation noise is modeled as multiplicative, using a lognormal distribution with $\sigma = 0.15$. Because of the multiplicative nature of the error in this strictly positive system, a log-transformation is applied to \cref{eqn_hes1} to produce Gaussian error distributions. Additionally, the component $H$ is treated as unobserved and excluded from the simulated observations. The absence of this component leads to a non-identifiable parameter inference problem (See \cref{fig:Hes1_CI} in Appendix).

Using the above setting, $2,000$ simulated datasets are generated. Three methods (SiGMoID, MAGI, and \cite{ramsay2007parameter}) are then implemented on each dataset to infer system solutions and estimate system parameters (\cref{fig4}). Among these methods, SiGMoID accurately infers both observable ($P, M$) and unobservable ($H$) true solutions. Further, compared to the two other methods, SiGMoID achieves the lowest RMSE values related to the difference between the inferred and true trajectories (\cref{table:RMSE_hes1}).

\begin{figure}[ht!]
\begin{center}
\centerline{\includegraphics[width=\columnwidth]{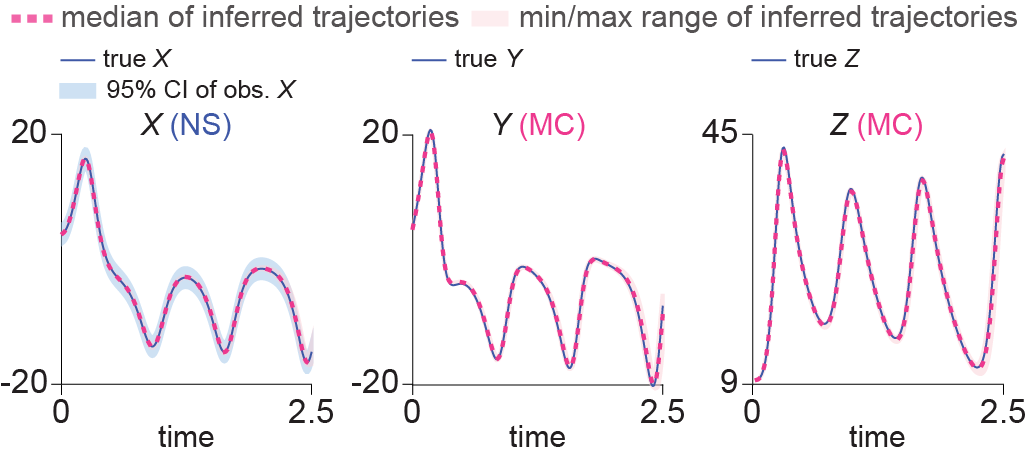}}
\caption{\textbf{System inference for a Lorenz equation.} We infer the true system solutions (true $X$, $Y$ and $Z$) using SiGMoID. (a) When all the components are given in the NS dataset, the solutions inferred using SiGMoID match the true solutions accurately. (b) Similar to (a), the solutions inferred using SiGMoID match the true solutions accurately, even though the components, $Y$ and $Z$ are not observable (MC).}
\label{fig5}
\end{center}
\end{figure}

We also compare the parameter estimates results in \cref{table:params_hes1} in Appendix. Note that while the temporal dynamics of the two observable components, $P$ and $M$, are influenced by various combinations of four parameters ($b, c, d$, and $e$), the unobservable component $H$ is determined exclusively by the other three parameters ($a, f$, and $g$), leading to an identifiability problem \cite{yang2021inference}. This shows that our estimation approach is well-suited for inferring $H$, as it can be accurately reconstructed once $a, f$, and $g$ are correctly determined.

\subsection{Lorenz system}
The Lorenz system was proposed to model simplified atmospheric convection using three key variables, $\mathbf{y} = (X, Y, Z)$ \cite{lorenz2017deterministic}. In this formulation, $X$ denotes the intensity of convective motion, $Y$ represents the temperature difference between rising and sinking air flows within a convection cell, and $Z$ indicates the deviation from thermal equilibrium or, more specifically, the vertical temperature distribution in the system:

\begin{equation*} f(\mathbf{y},\mathbf{p})=\begin{pmatrix} \sigma (Y-X)\\ X (\rho - Z) - Y \\ XY - \beta Z \end{pmatrix}, \end{equation*} 

where $\mathbf{p} = (\sigma, \rho, \beta)$ represents physical parameters with specific interpretations: the Prandtl number $\sigma$ controls the ratio of fluid viscosity to thermal diffusivity, the Rayleigh number $\rho$ quantifies the driving force of convection due to temperature differences, and the parameter $\beta$ is associated with the damping effect on convective motion.

We examine the capability of SiGMoID to estimate the parameters of the Lorenz system under the NSMC setting, where two system components, $Y$ and $Z$ are unobservable (\cref{fig5}). The Lorenz parameters are set to $\sigma = 10$, $\rho = 28$, and $\beta = 8/3$, with the initial condition set to $\mathbf{y}(0)=(4.67, 5.49, 9.06)$ \cite{hirsch2013differential}. 

In this setting, we first generate the NS dataset by integrating the Lorenz system using the RK4 method, adding additive noise with a noise level of 0.1. For the NSMC dataset, we omit both $Y$ and $Z$ variables. This setup is motivated by the framework introduced in \cite{stepaniants2024discovering}, and such missing components still introduce challenges for parameter identifiability (See \cref{fig:Lorenz_CI} in Appendix for details). Next, one hundred observed trajectories for $X$ is generated using 9 observation points.

Although $Y$ and $Z$ are remained unobserved (\cref{fig5}), SiGMoID accurately infers both unobservable states ($Y, Z$). Compared to the result of \cite{stepaniants2024discovering}, the estimation result obtained by SiGMoID shows significantly improved accuracy in capturing the difference between the inferred and true trajectories. Furthermore, the inferred parameters closely match the true values, as demonstrated in \cref{table:params_lorenz} in Appendix.

\section*{Disucssion}
SiGMoID enables accurate parameter estimation, robust noise quantification, and precise inference of unobserved system components across various synthetic examples by integrating HyperPINN with W-GAN. These experimental results indicate the potential applicability of existing deep learning-based models across diverse fields. In particular, the proposed framework shows strong potential to address real-world NSMC challenges. For example, in virology, only viral load may be observed in target-cell limited datasets \cite{smith2018influenza}; in epidemiological modeling, the number of individuals in the latent stage may not be recorded over time \cite{hong2024overcoming}; and in cell-signaling studies, intermediate steps in the pathway are often hidden \cite{jo2024density}. In such diverse fields, full system observations are rarely available, highlighting the potential of SiGMoID to contribute to data-driven scientific discovery under incomplete or partially observed conditions.

The main task of SiGMoID is not to develop a new surrogate model, but to reconstruct missing components from incomplete (NSMC) data by combining two modules: (i) a surrogation model (HyperPINN) and (ii) an inference module (W-GAN). While HyperPINN was used as the surrogate model in this study, this choice is not essential. Depending on the characteristics of the underlying dynamical system, other surrogate architectures such as Fourier neural operators \cite{li2020fourier} can be employed to further improve computational efficiency and scalability. This modular design underpins the generality of SiGMoID and supports its applicability to a broad class of dynamical systems beyond the four examples presented in this paper.

The use of a HyperPINN-based ODE solver significantly enhanced simulation efficiency by quickly generating solutions of differential equations across a wide range of parameter settings. This addresses the computational inefficiencies inherent in traditional PINN methods, offering scalability for high-dimensional or real-time system analysis. However, there are clear areas for improvement in this approach. Specifically, training HyperPINNs requires a priori knowledge of the distribution of parameters within the dataset \cite{de2021hyperpinn}. The number of artificial trajectories required for training grows proportionally with the size of the parameter space, leading to increased computational cost and training time. This limitation highlights an ongoing challenge in the field of deep learning-based operator learning. Consequently, the overall efficiency of SiGMoID can be further enhanced in parallel with advances in ODE solver methodologies (e.g., \cite{lee2025finite}).

Recently, advanced generative modeling techniques such as diffusion models and flow-matching methods \cite{ho2020denoising, lipman2022flow} have gained significant attention for their strong performance in high-dimensional data generation tasks. However, these approaches are primarily designed for data-to-data generation, where both input and output share the same sample space. In contrast, our task involves parameter-to-data inference, which requires learning a mapping from parameters in dynamic system to observed trajectories. Moreover, these models typically require large datasets and extensive training to perform well, which is not feasible in our data-scarce context. In many biological domains, particularly those involving living organisms, the number of available samples is inherently very limited. To reflect this realistic scenario, our simulations were conducted with fewer than one hundred observed trajectories. Therefore, we adopted W-GAN as a practical and stable framework for distribution matching under data-scarce conditions, and it demonstrated robust performance across our experiments, confirming its suitability for this study.

An important future direction is to extend the framework to settings where the governing equations are also partially or entirely unknown. While the current study focuses on parameter inference under known equation structures, which is a realistic scenario in many scientific domains, discovering the underlying dynamics itself remains a fundamental open challenge. Existing approaches such as SINDy and neural ODEs require some prior knowledge of functional forms or suffer from severe identifiability issues when multiple candidate dynamics are consistent with the same data \cite{champion2019data, chen2018neural}. Moreover, purely data driven methods are highly sensitive to noise and often lack interpretability, which is critical for scientific understanding. We have also conducted preliminary experiments with transformer based models for structure discovery \cite{d2023odeformer}. Although promising, these approaches remain sensitive to identifiability issues, particularly under limited and noisy data. Developing hybrid methods that integrate such structure discovery models with robust inference techniques, while incorporating domain specific priors or symbolic inductive biases, represents a promising but challenging avenue for future research.

\section*{Conclusion}
In this study, we introduce SiGMoID (Simulation-based Generative Model for Imperfect Data), a framework designed to handle noisy, sparse, and partially observed data in dynamic systems. By integrating HyperPINN with W-GAN, SiGMoID enables simultaneous inference of system parameters and reconstruction of missing dynamics, achieving superior accuracy and robustness compared to conventional methods \cite{dondelinger2013ode, wenk2019fast, yang2021inference, stepaniants2024discovering}. Therefore, these results show that deep learning–based inference models can be broadly applicable across various domains, even when full system observations are scarce. Consequently, SiGMoID offers a powerful approach for data-driven scientific discovery under incomplete or partially observed conditions.

\section{Acknowledgments}
Hyung Ju Hwang was supported by the National Research Foundation of Korea (NRF) grant funded by the Korea government (MSIT) (RS-2023-00219980 and RS-2022-00165268) and by Institute for Information \& Communications Technology Promotion (IITP) grant funded by the Korea government (MSIP) (No. RS-2019-II191906, Artificial Intelligence Graduate School Program (POSTECH)). Hyeontae Jo was supported by the National Research Foundation of Korea (RS-2024-00357912) and a Korea University grant (K2418321, K2425881).

\newpage
\bibliography{aaai2026}

@article{hirata2002oscillatory,
  title={Oscillatory expression of the bHLH factor Hes1 regulated by a negative feedback loop},
  author={Hirata, Hiromi and Yoshiura, Shigeki and Ohtsuka, Toshiyuki and Bessho, Yasumasa and Harada, Takahiro and Yoshikawa, Kenichi and Kageyama, Ryoichiro},
  journal={Science},
  volume={298},
  number={5594},
  pages={840--843},
  year={2002},
  publisher={American Association for the Advancement of Science}
}

@article{jo2024density,
  title={Density physics-informed neural networks reveal sources of cell heterogeneity in signal transduction},
  author={Jo, Hyeontae and Hong, Hyukpyo and Hwang, Hyung Ju and Chang, Won and Kim, Jae Kyoung},
  journal={Patterns},
  volume={5},
  number={2},
  year={2024},
  publisher={Elsevier}
}

@article{forger2024biological,
  title={Biological clocks, rhythms, and oscillations: the theory of biological timekeeping},
  author={Forger, Daniel B},
  year={2024},
  publisher={MIT Press}
}

@article{smith2018influenza,
  title={Influenza virus infection model with density dependence supports biphasic viral decay},
  author={Smith, Amanda P and Moquin, David J and Bernhauerova, Veronika and Smith, Amber M},
  journal={Frontiers in microbiology},
  volume={9},
  pages={1554},
  year={2018},
  publisher={Frontiers Media SA}
}

@article{jung2020real,
  title={Real-world implications of a rapidly responsive COVID-19 spread model with time-dependent parameters via deep learning: Model development and validation},
  author={Jung, Se Young and Jo, Hyeontae and Son, Hwijae and Hwang, Hyung Ju},
  journal={Journal of medical Internet research},
  volume={22},
  number={9},
  pages={e19907},
  year={2020},
  publisher={JMIR Publications Toronto, Canada}
}

@article{hong2024overcoming,
  title={Overcoming bias in estimating epidemiological parameters with realistic history-dependent disease spread dynamics},
  author={Hong, Hyukpyo and Eom, Eunjin and Lee, Hyojung and Choi, Sunhwa and Choi, Boseung and Kim, Jae Kyoung},
  journal={Nature Communications},
  volume={15},
  number={1},
  pages={8734},
  year={2024},
  publisher={Nature Publishing Group UK London}
}

@book{busenberg2012differential,
  title={Differential Equations and Applications in Ecology, Epidemics, and Population Problems},
  author={Busenberg, Stavros},
  year={2012},
  publisher={Elsevier}
}

@article{fitzhugh1961impulses,
  title={Impulses and physiological states in theoretical models of nerve membrane},
  author={FitzHugh, Richard},
  journal={Biophysical journal},
  volume={1},
  number={6},
  pages={445--466},
  year={1961},
  publisher={Elsevier}
}

@article{vyshemirsky2008bayesian,
  title={Bayesian ranking of biochemical system models},
  author={Vyshemirsky, Vladislav and Girolami, Mark A},
  journal={Bioinformatics},
  volume={24},
  number={6},
  pages={833--839},
  year={2008},
  publisher={Oxford University Press}
}

@article{ramsay2007parameter,
  title={Parameter estimation for differential equations: a generalized smoothing approach},
  author={Ramsay, Jim O and Hooker, Giles and Campbell, David and Cao, Jiguo},
  journal={Journal of the Royal Statistical Society Series B: Statistical Methodology},
  volume={69},
  number={5},
  pages={741--796},
  year={2007},
  publisher={Oxford University Press}
}

@article{calderhead2008accelerating,
  title={Accelerating Bayesian inference over nonlinear differential equations with Gaussian processes},
  author={Calderhead, Ben and Girolami, Mark and Lawrence, Neil},
  journal={Advances in neural information processing systems},
  volume={21},
  year={2008}
}

@inproceedings{dondelinger2013ode,
  title={ODE parameter inference using adaptive gradient matching with Gaussian processes},
  author={Dondelinger, Frank and Husmeier, Dirk and Rogers, Simon and Filippone, Maurizio},
  booktitle={Artificial intelligence and statistics},
  pages={216--228},
  year={2013},
  organization={PMLR}
}

@inproceedings{wenk2019fast,
  title={Fast Gaussian process based gradient matching for parameter identification in systems of nonlinear ODEs},
  author={Wenk, Philippe and Gotovos, Alkis and Bauer, Stefan and Gorbach, Nico S and Krause, Andreas and Buhmann, Joachim M},
  booktitle={The 22nd International Conference on Artificial Intelligence and Statistics},
  pages={1351--1360},
  year={2019},
  organization={PMLR}
}

@article{yang2021inference,
  title={Inference of dynamic systems from noisy and sparse data via manifold-constrained Gaussian processes},
  author={Yang, Shihao and Wong, Samuel WK and Kou, SC},
  journal={Proceedings of the National Academy of Sciences},
  volume={118},
  number={15},
  pages={e2020397118},
  year={2021},
  publisher={National Acad Sciences}
}

@article{ha2016hypernetworks,
  title={Hypernetworks},
  author={Ha, David and Dai, Andrew and Le, Quoc V},
  journal={arXiv preprint arXiv:1609.09106},
  year={2016}
}

@inproceedings{de2021hyperpinn,
  title={HyperPINN: Learning parameterized differential equations with physics-informed hypernetworks},
  author={de Avila Belbute-Peres, Filipe and Chen, Yi-fan and Sha, Fei},
  booktitle={The symbiosis of deep learning and differential equations},
  year={2021}
}

@article{lee2023hyperdeeponet,
  title={HyperDeepONet: learning operator with complex target function space using the limited resources via hypernetwork},
  author={Lee, Jae Yong and Cho, Sung Woong and Hwang, Hyung Ju},
  journal={arXiv preprint arXiv:2312.15949},
  year={2023}
}

@article{galanti2020modularity,
  title={On the modularity of hypernetworks},
  author={Galanti, Tomer and Wolf, Lior},
  journal={Advances in Neural Information Processing Systems},
  volume={33},
  pages={10409--10419},
  year={2020}
}

@article{ramesh2022gatsbi,
  title={GATSBI: Generative adversarial training for simulation-based inference},
  author={Ramesh, Poornima and Lueckmann, Jan-Matthis and Boelts, Jan and Tejero-Cantero, {\'A}lvaro and Greenberg, David S and Gon{\c{c}}alves, Pedro J and Macke, Jakob H},
  journal={arXiv preprint arXiv:2203.06481},
  year={2022}
}

@article{gloeckler2024all,
  title={All-in-one simulation-based inference},
  author={Gloeckler, Manuel and Deistler, Michael and Weilbach, Christian and Wood, Frank and Macke, Jakob H},
  journal={arXiv preprint arXiv:2404.09636},
  year={2024}
}

@inproceedings{arjovsky2017wasserstein,
  title={Wasserstein generative adversarial networks},
  author={Arjovsky, Martin and Chintala, Soumith and Bottou, L{\'e}on},
  booktitle={International conference on machine learning},
  pages={214--223},
  year={2017},
  organization={PMLR}
}

@article{gulrajani2017improved,
  title={Improved training of wasserstein gans},
  author={Gulrajani, Ishaan and Ahmed, Faruk and Arjovsky, Martin and Dumoulin, Vincent and Courville, Aaron C},
  journal={Advances in neural information processing systems},
  volume={30},
  year={2017}
}

@book{villani2009optimal,
  title={Optimal transport: old and new},
  author={Villani, C{\'e}dric and others},
  volume={338},
  year={2009},
  publisher={Springer}
}

@article{goodfellow2014generative,
  title={Generative adversarial nets},
  author={Goodfellow, Ian and Pouget-Abadie, Jean and Mirza, Mehdi and Xu, Bing and Warde-Farley, David and Ozair, Sherjil and Courville, Aaron and Bengio, Yoshua},
  journal={Advances in neural information processing systems},
  volume={27},
  year={2014}
}

@article{kadeethum2021framework,
  title={A framework for data-driven solution and parameter estimation of pdes using conditional generative adversarial networks},
  author={Kadeethum, Teeratorn and O’Malley, Daniel and Fuhg, Jan Niklas and Choi, Youngsoo and Lee, Jonghyun and Viswanathan, Hari S and Bouklas, Nikolaos},
  journal={Nature Computational Science},
  volume={1},
  number={12},
  pages={819--829},
  year={2021},
  publisher={Nature Publishing Group US New York}
}

@article{patel2022solution,
  title={Solution of physics-based Bayesian inverse problems with deep generative priors},
  author={Patel, Dhruv V and Ray, Deep and Oberai, Assad A},
  journal={Computer Methods in Applied Mechanics and Engineering},
  volume={400},
  pages={115428},
  year={2022},
  publisher={Elsevier}
}

@article{mirza2014conditional,
  title={Conditional generative adversarial nets},
  author={Mirza, Mehdi and Osindero, Simon},
  journal={arXiv preprint arXiv:1411.1784},
  year={2014}
}

@article{samson2025inference,
  title={Inference for the stochastic FitzHugh-Nagumo model from real action potential data via approximate Bayesian computation},
  author={Samson, Adeline and Tamborrino, Massimiliano and Tubikanec, Irene},
  journal={Computational Statistics \& Data Analysis},
  volume={204},
  pages={108095},
  year={2025},
  publisher={Elsevier}
}

@article{cho2024estimation,
  title={Estimation of System Parameters Including Repeated Cross-Sectional Data through Emulator-Informed Deep Generative Model},
  author={Cho, Hyunwoo and Cho, Sung Woong and Jo, Hyeontae and Hwang, Hyung Ju},
  journal={arXiv preprint arXiv:2412.19517},
  year={2024}
}

@article{li1996simultaneous,
  title={Simultaneous approximations of multivariate functions and their derivatives by neural networks with one hidden layer},
  author={Li, Xin},
  journal={Neurocomputing},
  volume={12},
  number={4},
  pages={327--343},
  year={1996},
  publisher={Elsevier}
}

@incollection{lorenz2017deterministic,
  title={Deterministic Nonperiodic Flow 1},
  author={Lorenz, Edward N},
  booktitle={Universality in Chaos, 2nd edition},
  pages={367--378},
  year={2017},
  publisher={Routledge}
}

@book{hirsch2013differential,
  title={Differential equations, dynamical systems, and an introduction to chaos},
  author={Hirsch, Morris W and Smale, Stephen and Devaney, Robert L},
  year={2013},
  publisher={Academic press}
}

@article{stepaniants2024discovering,
  title={Discovering dynamics and parameters of nonlinear oscillatory and chaotic systems from partial observations},
  author={Stepaniants, George and Hastewell, Alasdair D and Skinner, Dominic J and Totz, Jan F and Dunkel, J{\"o}rn},
  journal={Physical Review Research},
  volume={6},
  number={4},
  pages={043062},
  year={2024},
  publisher={APS}
}

@article{lee2025finite,
  title={Finite Element Operator Network for Solving Elliptic-Type Parametric PDEs},
  author={Lee, Jae Yong and Ko, Seungchan and Hong, Youngjoon},
  journal={SIAM Journal on Scientific Computing},
  volume={47},
  number={2},
  pages={C501--C528},
  year={2025},
  publisher={SIAM}
}

@article{li2020fourier,
  title={Fourier neural operator for parametric partial differential equations},
  author={Li, Zongyi and Kovachki, Nikola and Azizzadenesheli, Kamyar and Liu, Burigede and Bhattacharya, Kaushik and Stuart, Andrew and Anandkumar, Anima},
  journal={arXiv preprint arXiv:2010.08895},
  year={2020}
}

@article{ho2020denoising,
  title={Denoising diffusion probabilistic models},
  author={Ho, Jonathan and Jain, Ajay and Abbeel, Pieter},
  journal={Advances in neural information processing systems},
  volume={33},
  pages={6840--6851},
  year={2020}
}

@article{lipman2022flow,
  title={Flow matching for generative modeling},
  author={Lipman, Yaron and Chen, Ricky TQ and Ben-Hamu, Heli and Nickel, Maximilian and Le, Matt},
  journal={arXiv preprint arXiv:2210.02747},
  year={2022}
}

@article{champion2019data,
  title={Data-driven discovery of coordinates and governing equations},
  author={Champion, Kathleen and Lusch, Bethany and Kutz, J Nathan and Brunton, Steven L},
  journal={Proceedings of the National Academy of Sciences},
  volume={116},
  number={45},
  pages={22445--22451},
  year={2019},
  publisher={National Academy of Sciences}
}

@article{chen2018neural,
  title={Neural ordinary differential equations},
  author={Chen, Ricky TQ and Rubanova, Yulia and Bettencourt, Jesse and Duvenaud, David K},
  journal={Advances in neural information processing systems},
  volume={31},
  year={2018}
}

@article{d2023odeformer,
  title={Odeformer: Symbolic regression of dynamical systems with transformers},
  author={d'Ascoli, St{\'e}phane and Becker, S{\"o}ren and Mathis, Alexander and Schwaller, Philippe and Kilbertus, Niki},
  journal={arXiv preprint arXiv:2310.05573},
  year={2023}
}

@article{raue2009structural,
  title={Structural and practical identifiability analysis of partially observed dynamical models by exploiting the profile likelihood},
  author={Raue, Andreas and Kreutz, Clemens and Maiwald, Thomas and Bachmann, Julie and Schilling, Marcel and Klingm{\"u}ller, Ursula and Timmer, Jens},
  journal={Bioinformatics},
  volume={25},
  number={15},
  pages={1923--1929},
  year={2009},
  publisher={Oxford University Press}
}

\newpage
\appendix
\onecolumn

\renewcommand{\thetable}{A\arabic{table}}
\renewcommand{\thefigure}{A\arabic{figure}}
\setcounter{table}{0}
\setcounter{figure}{0}

\section{Methodological Details}\label{method_detail}
This section presents the structure of SiGMoID in detail and describes its inference strategy used to capture system parameters $\mathbf{p}$ in \cref{general_ode} and noise level $\mathbf{\epsilon}$ in \cref{data_noise}, which enables the recovery of missing components in NSMC data. To describe SiGMoID structurally, we first formalize the representation of imperfect data (NS and NSMC data) in \cref{data_noise} as follows. Let $I$ and $J$ be the index sets of all observable components of $\mathbf{y}$ and observation time points $t^{o}_{j}$ respectively. While the observation time points may differ across components in general (i.e., the number of observation points $N_{o}$ and the values $t_{j}^{o}$ can depend on $y_{i}$), all observation time points can be merged into a single ordered set by combining the unique time points across all components. Hence we assume that this unified ordered set is used for the analysis. Then the observed data can be described as follows:
\begin{eqnarray}\label{data_noise_math}
    Y_{n}^{o}=& \{(t_{j},y_{n,i}^{o}(t_{j}))\}_{i \in I, j \in J} \nonumber\\
    =& \{(t_{j}, y_{i}(t_{j};\mathbf{p}_{true})+e_{n,i}^{o}(t_{j}))\}_{i \in I, j \in J}
\end{eqnarray}
where $e_{n,i}^{o}(t_{j})$ denotes the noise for $n$-th observed data $Y_{n}^{o}$ at $t_{j}$. For convenience, we define the set of all noise terms in $n$-th data, $\mathbf{e}_{n}^{o}=\{e_{n,i}^{o}(t_{j})\}_{i\in I, j\in J}$. If $\mathbf{e}_{n}^{o}$ has a nonzero element, $Y^{o}=\{Y_{n}^{o}\}_{n=1}^{N_{o}}$ is called NS; if there exists a component $i \in \{1,\dots,d_{y}\}$ such that $i\notin I$, then $Y^{o}$ is called NSMC. Notably, both forms share the common characteristic of lacking a complete set of observations, with missing components potentially regarded as extreme cases of sparse data. The primary distinction between the two is that sparse data, for which some observations are available, can often be addressed using simple interpolation. In contrast, the recovery of missing components is more challenging, especially when the information available about system dynamics is insufficient. In other words, NSMC data can be treated as a subcategory of NS data.

\subsection{ODE solver using HyperPINN}
We propose an ODE solver capable of directly computing the solution $\mathbf{y}(t)$ of \cref{general_ode} for a given parameter $\mathbf{p}$ using HyperPINN \cite{de2021hyperpinn}, which incorporates two fully connected neural networks---a hypernetwork $h$ \cite{ha2016hypernetworks, galanti2020modularity, lee2023hyperdeeponet} and a fully connected main network $m$. Specifically, the hypernetwork $h$, parameterized by weights and biases, $\theta_{h}$, maps the parameter $\mathbf{p}$ to the weights and biases $\theta_{m}$ of the main network $m$: 
\begin{equation} \label{equation:hypernetwork} \theta_{m}(\mathbf{p})=h(\mathbf{p};\theta_h).
\end{equation} 
Details regarding the number of nodes and activation functions are provided in \cref{table:network settings:hyperpinn} in Appendix. Once $\theta_{m}(\mathbf{p})$ is obtained by training the hypernetwork, these values are used as the weights and biases of the main network $m$. In other words, the output of the main network $m$ is directly determined by the outputs of the hypernetwork. Consequently, the main network immediately yields a function $m(t;\theta_m(\mathbf{p}))$ that closely approximates the solution of \cref{general_ode}, $\mathbf{y}(t;\mathbf{p})$:
\begin{equation}\label{equation:mainnetwork}
\mathbf{y}(t;\mathbf{p}) \approx m(t;\theta_m(\mathbf{p})).
\end{equation}
To train HyperPINN for the aforementioned task, we first define the probability distribution of parameters $\mathbf{p}$, denoted by $\mathcal{D}$, as well as the time interval $[0, T]$, which encompasses the time span covered by experimental data. Next, we construct two loss functions based on \citep{de2021hyperpinn}: 1) data loss $L_{\text{data}}$ and 2) physics loss $L_{\text{physics}}$. First, $L_{\text{data}}$ is used to fit the output of the main network $m$ to the solution $\mathbf{y}(t;\mathbf{p})$ by minimizing the difference between them, as follows:
\begin{equation}\label{equation:loss_data}
L_{\text{data}}(\theta_h) = \sum_{j=1}^{T_{col}} \mathbb{E}_{\mathbf{p}\sim\mathcal{D}} \left\Vert m(t^c_{j}; \theta_m(\mathbf{p})) - \mathbf{y}(t^c_{j};\mathbf{p})\right\Vert^2,
\end{equation}
where $\{t_j^{c}\}_{j=1}^{T_{col}}$ represents the collocation time points within the time interval $[0,T]$ and $\mathbb{E}$ represents the expectation over the probability distribution $\mathcal{D}$. Simultaneously, $L_{\text{physics}}$ is designed to measure the extent to which the output of the main network $m$ satisfies the ODEs \cref{general_ode}. This measurement can be quantified by substituting the output of the main network into the DE, as follows:
\begin{equation}\label{equation:loss_phy}
L_{\text{physics}}(\theta_h) = \sum_{j=1}^{T_{col}} \mathbb{E}_{\mathbf{p}\sim\mathcal{D}} 
\Bigg\Vert \frac{d}{dt} m(t_j^{c}; \theta_m(\mathbf{p}))
- \mathbf{f}\big[m(t_j^{c}; \theta_m(\mathbf{p})), \mathbf{p} \big] \Bigg\Vert^2.
\end{equation}
We choose $\mathcal{D}$ to be a uniform distribution over the interval $[\mathbf{p}_{min}, \mathbf{p}_{max}]$, denoted as $\mathcal{U}_{[\mathbf{p}_{min}, \mathbf{p}_{max}]}$. The lower and upper bounds of the interval are typically determined empirically. In this paper, we choose $0.5*\mathbf{p}^{true}$ and $1.5*\mathbf{p}^{true}$ as the respective bounds in all our examples, with the true underlying parameter $\mathbf{p}^{true}$. Subsequently, we calculate $N_{p}$ numerical solutions $\{\mathbf{y}_{i}^{s}(t_{j}^{c};\mathbf{p}^{s}_{i}) | \mathbf{p}^{s}_{i}\sim\mathcal{D}\}_{i=1}^{N_{p}}$ for each collocation time points $\{t_{j}^{c}\}_{j=1}^{T_{col}}$ to construct the training dataset. By minimizing \cref{equation:loss_data} and \cref{equation:loss_phy}, the output of the main network $m(t;\theta_m(\mathbf{p}))$ is expected to not only closely approximate $\mathbf{y}(t;\mathbf{p})$ but also accurately describe the underlying dynamics of the system. Through this training procedure, we expect that the main network $m$ immediately provides solutions corresponding to arbitrary $\mathbf{p}\in\mathcal{D}$. The theoretical justification of this procedure is provided in Mathematical Analysis.

Because the units of the two loss functions are generally not equal, biases may arise while minimizing them (\cref{equation:loss_data}-\cref{equation:loss_phy}). To avoid this, two positive weights, $\alpha$ and $\beta$, are introduced. This enables training to be performed by minimizing the total loss function, $L(\theta_h)$, which is defined as the weighted sum of the two loss functions using $\alpha$ and $\beta$, as follows:
\begin{equation}\label{equation:total}
L(\theta_h) = \alpha L_{\text{data}}(\theta_h) + \beta L_{\text{physics}}(\theta_h).
\end{equation} Detailed values for selecting $N_p, \alpha$, and $\beta$ are provided in  \cref{table:network settings:hyperpinn}.

\subsection{Estimation of NS and NSMC data with W-GAN}
In this section, we demonstrate the estimation of the true parameter $\mathbf{p}^{true}$ and data noise $\mathbf{e}^{o}=\{\mathbf{e}_{n}^{o}\}_{n=1}^{N_{o}}$ that generates the given data $Y^{o}$ (defined in Equation~(4)) using HyperPINN. To this end, we construct W-GAN comprising two generators, $G^p$ and $G^e$, and a discriminator, $D$ with trainable parameters $\theta_{G^p}, \theta_{G^e},$ and $\theta_{D}$, respectively. For the two generators, we randomly sample two sets of variables $\{ \mathbf{z}^{p}_{n} \}_{n=1}^{N_{G}}$, $\{\mathbf{z}^{e}_{n} \}_{n=1}^{N_{G}}$ from the standard normal distribution $\mathcal{N}(0,I_{d_{n}})$, where $d_{n}$ denotes the noise dimension. These samples are mapped to $\mathbf{p}^{G}=\{\mathbf{p}_{n}^{G}=G^p(\mathbf{z}_{n}^{p})\}_{n=1}^{N_{G}}$ and $\mathbf{e}^{G}=\{\mathbf{e}_{n}^{G}=G^{e}(\mathbf{z}_{n}^{e})\}_{n=1}^{N_{G}}$ via the two generators, $G^p$ and $G^e$, of the W-GAN. These mappings may be interpreted as sets of possible true parameters and noise, respectively. Let the transformed variables follow the distributions, $\pi_{p}^{G}(\mathbf{p}_{n}^{G})$ and $\pi_{e}^{G}(\mathbf{e}_{n}^{G})$. For each ($\mathbf{p}_{n}^{G}$,$\mathbf{e}_{n}^{G}$), we can simultaneously obtain approximations for the solutions of ~(\cref{general_ode}), $Y_{n}^{G}$, using HyperPINN.\begin{eqnarray*}
    Y_{n}^{G}=& \{(t_{j},y_{n,i}^{G}(t_{j}))\}_{i \in I, j \in J} \nonumber\\
    =& \{(t_{j}, m_{i}(t_{j};\theta_{m}(\mathbf{p}_{n}^{G}))+e_{n,i}^{G}(t_{j}))\}_{i \in I, j \in J}
\end{eqnarray*}
where $e_{n,i}^{G}(t_{j})$ denotes the generated noise for $n$-th generated trajectory $Y_{n}^{G}$ at $t_{j}$. For convenience, we define the set of all noise terms in $n$-th trajectory, $\mathbf{e}_{n}^{G}=\{e_{n,i}^{G}(t_{j})\}_{i\in I, j\in J}$. Next, we adjust the distributions $\pi_{p}^{G}(\mathbf{p}_{n}^{G}), \pi_{e}^{G}(\mathbf{e}_{n}^{G})$ to ensure that the distribution of $Y_{n}^{G}$ is sufficiently close to the given observed dataset $Y_{n}^{o}$. This adjustment involves measuring the difference between $Y^{G}=\{Y_{n}^{G}\}_{n=1}^{N_{G}}$ and $Y^{o}=\{Y_{n}^{o}\}_{n=1}^{N_{o}}$, and then modifying the sets $\mathbf{p}^{G}$, $\mathbf{e}^{G}$ to minimize this difference. Note that in practice, we set $N_{G}=N_{o}$.

For this task, we utilize the discriminator in W-GAN equipped with a gradient penalty, yielding the correct parameter distribution via a generator \citep{arjovsky2017wasserstein, gulrajani2017improved}. Using W-GAN, we aim to reduce the Wasserstein distance (with Kantorovich–Rubinstein duality \citep{villani2009optimal}) between the distribution of $Y^{G}$ and $Y^{o}$, denoted as $\mu_{G}$ and $\mu_{o}$, respectively: 
\begin{equation}\label{w_dist} 
d(\mu_{G}, \mu_{o}) = \sup_{\phi \in \text{Lip}_{1}}\mathbb{E}_{Y^{G} \sim \mu_{G}}[\phi(Y^{G})]-\mathbb{E}_{Y^{o} \sim \mu_{o}}[\phi(Y^{o})],
\end{equation}
where $\text{Lip}_{1} = \{\phi \in Lip(\mathbb{R}^{|\Omega|},\mathbb{R}) | \| \phi\|_{Lip}\leq1\}$ represents the Lipschitz constant over the set of all real-valued Lipschitz functions on $\mathbb{R}^{|\Omega|}$ (i.e., $Lip(\mathbb{R}^{|\Omega|},\mathbb{R})=\{f:\mathbb{R}^{|\Omega|}\rightarrow\mathbb{R} | \|f\|_{Lip}=\sup_{x \neq y } \frac{|f(x)-f(y)|}{|x-y|}<\infty $\}). Ideally, we obtain both true parameters and data noise by minimizing \cref{w_dist} via ADAM optimizer. Futhermore, to ensure stable training, we monitor the approximated Wasserstein distance during optimization and apply early stopping once the distance ceases to decrease. In practice, this criterion leads to convergence within 100,000 epochs. 

However, \cref{w_dist} cannot be directly applied to computational devices, as it requires searching over infinite-dimensional spaces (e.g., the $\sup$ operation). To address this, a suitable approximation of \cref{w_dist} are needed for practical implementation on computational devices. For this, we utilize the discriminator $D(\cdot;\theta_{D})$ in W-GAN, with a weights and biases $\theta_{D}$. We then alternatively use the following two loss functions---generator loss $L_{G}(\theta_{G},\theta_{D})$ and discriminator loss $L_{D}(\theta_{G},\theta_{D})$---instead of the \cref{w_dist}: 
\begin{align}
\displaystyle L_{G}(\theta_{G},\theta_{D}) & = \frac{1}{N_{o}}\sum_{n=1}^{N_{o}}D(Y_{n}^{o};\theta_{D})\nonumber - \frac{1}{N_{G}}\sum_{n=1}^{N_{G}}D(Y_{n}^{G};\theta_{D}) + \lambda_{e}L_{e}(\theta_{G^{e}}), \end{align}\begin{align}
\displaystyle
L_{e}(\theta_{G^{e}}) &= \sum_{i \in I, j \in J }\frac{1}{N_{G}}\sum_{n=1}^{N_{G}}e_{n,i}^{G}(t_{j})^{2} + \sum_{i \in I, j \in J }\frac{1}{N_{G}}\sum_{n=1}^{N_{G}}(\text{Var}(y_{ij}^{o}) - \text{Var}(e_{n,i}^{G}(t_{j})))^{2},\nonumber
\end{align}
\begin{align}
\displaystyle L_{D}(\theta_{G},\theta_{D}) & = -\frac{1}{N_{o}}\sum_{n=1}^{N_{o}}D(Y_{n}^{o};\theta_{D})\nonumber + \frac{1}{N_{G}}\sum_{n=1}^{N_{G}}D(Y_{n}^{G};\theta_{D}) \\
& + \frac{\lambda_{D} }{N_{G}}\sum_{n=1}^{N_{G}}(||\nabla_{\hat{Y}_{n}^{G}}D(\hat{Y}_{n}^{G};\theta_{D})||_{2}-1)^{2},\nonumber
\end{align}
where $\theta_{G} = (\theta_{G^{p}},\theta_{G^{e}})$. Here,  the term $L_{e}$ of generator loss constrains the mean and variance of generated data noise $\mathbf{e}^{G}$ to improve the accuracy of the approximation task of  $G^{e}$, as theoretically justified in Mathematical Analysis \cref{gan3}. The coefficient $\lambda_{e}$ is chosen to minimize the RMSE computed on the observable trajectories.

The augmented data $\displaystyle \hat{Y}^{G}=\{(t_{j},\hat{y}_{n,i}^{G}(t_{j}))\}_{i \in I, j \in J}$ for calculating the gradient penalty (in the last term of $\displaystyle L_{D}(\theta_{G},\theta_{D})$) is defined as:
$$
\displaystyle \hat{y}_{n,i}^{G}(t_{j}) = \gamma(i,t_{j})y_{n,i}^{G}(t_{j}) + (1-\gamma(i,t_{j})) y_{n,i}^{o}(t_{j}),
$$
with uniform random coefficient $\gamma(i,t_{j}) \sim U[0,1]$ for each $i \in I, j \in J$. The gradient penalty term enforces the discriminator $D$ to be a 1-Lipschitz function, ensuring that minimizing the loss is equivalent to finding $\phi$ in \cref{w_dist} \citep{gulrajani2017improved}. Following the recommendations of \citep{gulrajani2017improved}, we set the value of the coefficient $\lambda_{D}$ to 10.

For each generator and discriminator in W-GAN, we employed fully connected neural networks with hyperbolic tangent activation functions (see also the hyperparameters for W-GAN in \cref{table:network settings:wgan} in Supplementary Tables of the Appendix). Given the scarcity of data, mini-batch training can introduce significant instability. Therefore, we used full batch training, generating fake data in the same quantity as the real data to enhance stability.
\begin{algorithm}\caption{Training process for SiGMoiD}\label{alg:SiGMoiD}
\begin{algorithmic}
\Require Gradient penalty coefficient $\lambda_{D}$, number of critic iterations per generator iteration $n_{\text{critic}}$, noise penalty coefficient $\lambda_{e}$, and Adam hyperparameters $\alpha$, $\beta_1$, $\beta_2$.
\Require Pretrained hypernetwork $h(\cdot; \theta_{h})$ and main network $m = \{m_{i}(\cdot; \theta_{m})\}_{i \in I}$.
\Require Initialize critic parameters $\theta_{D}$ and generator parameters $\theta_{G} = (\theta_{G^{p}}, \theta_{G^{e}})$.
\While{$\theta_{G}$ has not converged}
    \For{$t = 1, \dots, n_{\text{critic}}$}
        \State Sample real trajectories $Y_{n}^{o} \sim \mathbb{P}_{r}$ and random numbers $\gamma(i,t_{j}) \sim U[0,1]$ for each component $i \in I$ and observed time $t_{j}$. Also, sample latent variables $\boldsymbol{z}_{n}^{p},\boldsymbol{z}_{n} ^{e} \sim p(z)$.
        \State $\boldsymbol{p}_{n}^{G} \leftarrow G^{p}(\boldsymbol{z}_{n}^{p})$
        \State $e_{n,i}^{G}(t_{j}) \leftarrow G^{e}(t_{j},\boldsymbol{z}_{n}^{e})$
        \State $y_{n,i}^{G}(t_{j}) \leftarrow m_{i}(t_{j};\theta_{m}(\boldsymbol{p}_{n}^{G}))+e_{n,i}^{G}(t_{j})$
        \State $\hat{y}_{n,i}^{G}(t_{j}) \leftarrow \gamma(i,t_{j}) y_{n,i}^{o}(t_{j}) + (1-\gamma(i,t_{j}))y_{n,i}^{G}(t_{j})$
        \State $L_{D}(\theta_{G},\theta_{D}) \leftarrow -\frac{1}{N_{o}}\sum_{n=1}^{N_{o}}D(Y_{n}^{o};\theta_{D})+\frac{1}{N_{G}}\sum_{n=1}^{N_{G}}D(Y_{n}^{G};\theta_{D})+\frac{\lambda_{D}}{N_{G}}(||\nabla_{\hat{Y}_{n}^{G}}D(\hat{Y}_{n}^{G};\theta_{D})||_{2}-1)^{2}$
        \State $\theta_{D} \leftarrow \text{Adam}(L_{D}(\theta_{G},\theta_{D}),\alpha,\beta_{1},\beta_{2}))$
    \EndFor
    \State Sample latent variables $\boldsymbol{z}_{n}^{p},\boldsymbol{z}_{n} ^{e} \sim p(z)$.
    \State $\boldsymbol{p}_{n}^{G} \leftarrow G^{p}(\boldsymbol{z}_{n}^{p})$
    \State $e_{n,i}^{G}(t_{j}) \leftarrow G^{e}(t_{j},\boldsymbol{z}_{n}^{e})$
    \State $y_{n,i}^{G}(t_{j}) \leftarrow m_{i}(t_{j};\theta_{m}(\boldsymbol{p}_{n}^{G}))+e_{n,i}^{G}(t_{j})$
     \State $L_{e}(\theta_{G^{e}}) \leftarrow \frac{1}{|I||J|}\sum_{i \in I, j \in J }\frac{1}{N_{G}}\sum_{n=1}^{N_{G}}e_{n,i}^{G}(t_{j})^{2} + \frac{1}{|I||J|}\sum_{i \in I, j \in J }(\text{Var}(y_{ij}^{o}) - \text{Var}(\{e_{n,i}^{G}(t_{j})\}_{n=1}^{N_{G}}))^{2}$
    \State $L_{G}(\theta_{G},\theta_{D}) \leftarrow -\frac{1}{N_{G}}\sum_{n=1}^{N_{G}}D(Y_{n}^{G};\theta_{D})+\lambda_{e}L_{e}(\theta_{G^{e}})$
    \State $\theta_{G} \leftarrow \text{Adam}(L_{G}(\theta_{G},\theta_{D}),\alpha,\beta_{1},\beta_{2}))$
\EndWhile
\end{algorithmic}
\end{algorithm}
\section{Mathematical Analysis}\label{math_analysis}
\subsection{Convergence of SiGMoID}
This section provides a mathematical analysis of the proposed model, SiGMoID, which integrates (1) HyperPINN and (2) W-GAN. For HyperPINN, we assume that the system function $\mathbf{f}(\mathbf{y},\mathbf{p})$  defined in Equation (1) is Lipschitz continuous with respect to $\mathbf{y}$ and $\mathbf{p}$. That is, there exists two positive constants $C_1$ and $C_2$ such that
$$\|\mathbf{f}(\mathbf{y}_1,\mathbf{p})-\mathbf{f}(\mathbf{y}_2,\mathbf{p})\|\leq C_{1}\|\mathbf{y}_{1}-\mathbf{y}_2\|,$$
$$\|\mathbf{f}(\mathbf{y},\mathbf{p}_1)-\mathbf{f}(\mathbf{y},\mathbf{p}_2)\|\leq C_{2}\|\mathbf{p}_{1}-\mathbf{p}_2\|,$$
for all $\mathbf{y}_1, \mathbf{y}_2 \in \mathbb{R}^{d_y}$ and  $\mathbf{p}_1, \mathbf{p}_2 \in \mathbb{R}^{d_p}$.

Under this assumption, \textbf{Theorem A.1} in \cite{cho2024estimation} shows that the solution generated by HyperPINN closely approximates the solution to Equation (1), with an error proportional to the total loss, defined in Equation (9), for a given set of parameters $\mathbf{p}$. Furthermore, the authors also showed that trainable parameters in HyperPINN can always be found to satisfy the following inequality (\textbf{Proposition A.3} and \textbf{Proposition A.5} in \cite{cho2024estimation}):

\begin{equation}\nonumber
L_{\text{physics}}(\theta_h)<\varepsilon,
\end{equation}

for a arbitrary constant $\varepsilon>0$. From the universal approximation theorem (e.g., \textbf{Theorem 2.1} in \cite{li1996simultaneous}), the total loss function can takes arbitrary small values by adjusting the trainable parameters in HyperPINN. 

Note that \textbf{Proposition A.3} states that realization of Equation (8) can be conducted by sampling $N_p$ parameters from the distribution $\mathcal{D}$. Although such discretization, it is shown that the error introduced by this discretization becomes negligible if $N_p$ is sufficiently large.

Next, we demonstrate that the solution estimates on observation time points from W-GAN can closely approximate the true solutions as the Wasserstein distance (defined in Equation (10)) converges to zero. 

\begin{proposition}\label{gan2} Suppose that $\pi_{e}^{G},\pi_{e}^{o}$ satisfy \[ \mathbb{E}_{\mathbf{e}_{n}^{G} \sim \pi_{e}^{G}}[\{e_{n,i}^{G}(t_{j})\}_{i\in I, j \in J}] = \mathbb{E}_{\mathbf{e}_{n}^{o} \sim \pi_{e}^{o}}[\{e_{n,i}^{o}(t_{j})\}_{i\in I, j \in J}] = \mathbf{0}.\] Then, if $\pi_{p}^{G}$, $\pi_{e}^{G}$ are the desired probability density functions (pdfs) that make Wasserstein distance $d(Y^{o},Y^{G})$ equal to zero, i.e.,
\begin{align*}
\sup_{\phi \in \mathrm{Lip}_{1}}\mathbb{E}_{\mathbf{p}_{n}^{G} \sim \pi_{p}^{G}, \mathbf{e}_{n}^{G} \sim \pi_{e}^{G}}\bigg[\phi (\{y_{i}(t_{j};\mathbf{p}_{n}^{G})&+e_{n,i}^{G}(t_{j})\}_{i\in I,j\in J})\bigg]\\
& - \mathbb{E}_{\mathbf{e}_{n}^{o} \sim \pi_{e}^{o}}\bigg[\phi (\{y_{i}(t_{j};\mathbf{p}_{true})+e_{n,i}^{o}(t_{j})\}_{i\in I,j\in J}) \bigg] = 0,
\end{align*}
 then \[\mathbb{E}_{\mathbf{p}_{n}^{G} \sim \pi_{p}^{G}}[y_{i}(t_{j};\mathbf{p}_{n}^{G})] = y_{i}(t_{j};\mathbf{p}_{true})\] and \[ \mathbb{E}_{\mathbf{p}_{n}^{G} \sim \pi_{p}^{G}}[(y_{i}(t_{j};\mathbf{p}_{n}^{G})-y_{i}(t_{j};\mathbf{p}_{true}))^{2}] = \text{Var}(e_{n,i}^{o}(t_{j}))-\text{Var}(e_{n,i}^{G}(t_{j}))\] for all $i \in I$, $j \in J$.
\end{proposition}
\begin{proof}We have that
\begin{align*}
0 &= \int\int \phi(\{y_{i}(t_{j};\mathbf{p}_{n}^{G})+e_{n,i}^{G}(t_{j})\})\pi_{p}^{G}(\mathbf{p}_{n}^{G})
\pi_{e}^{G}(\mathbf{e}_{n}^{G})d\mathbf{p}_{n}^{G}d\mathbf{e}_{n}^{G}\\
 & \quad\quad -\int \phi(\{y_{i}(t_{j};\mathbf{p}_{true})+e_{n,i}^{o}(t_{j})\})\pi_{e}^{o}(\mathbf{e}_{n}^{o})d\mathbf{e}_{n}^{o} \\ 
&= \int\int \phi(x)\pi_{p}^{G}(\mathbf{p}_{n}^{G})\pi_{e}^{G}(\{x_{i,j}-y_{i}(t_{j};\mathbf{p}_{n}^{G})\})d\mathbf{p}_{n}^{G}dx-
\int \phi(x)\pi_{e}^{o}(\{x_{i,j}-y_{i}(t_{j};\mathbf{p}_{true})\})dx \\ &=\int \phi(x)\Bigg(
\int \pi_{e}^{G}(\{x_{i,j}-y_{i}(t_{j};\mathbf{p}_{n}^{G})\})\pi_{p}^{G}(\mathbf{p}_{n}^{G})d\mathbf{p}_{n}^{G}-\pi_{e}^{o}(\{x_{i,j}-y_{i}(t_{j};\mathbf{p}_{true})\})\Bigg)dx \end{align*} for all $\phi \in \text{Lip}_{1}$
and thus we obtain 
\[\mathbb{E}_{\mathbf{p}_{n}^{G} \sim \pi_{p}^{G}}[\pi_{e}^{G}(\{x_{i,j}-y_{i}(t_{j};\mathbf{p}_{n}^{G})\})]=\pi_{e}^{o}(\{x_{i,j}-y_{i}(t_{j};\mathbf{p}_{true})\})\] for any $x_{i,j} \in \mathbb{R}$. \\
Now denote that $\pi_{i,j}^{G}$, $\pi_{i,j}^{o}$ be marginal pdfs for $e_{n,i}^{G}(t_{j})$, $e_{n,i}^{o}(t_{j})$ respectively.

From the given condition, we have that 
\begin{align*}
    \int x_{i,j}\pi_{e}^{G}(\{x_{i',j'}-y_{i'}(t_{j'};\mathbf{p}_{n}^{G})\})dx = y_{i}(t_{j};\mathbf{p}_{n}^{G}),\\ 
\int x_{i,j}\pi_{e}^{o}(\{x_{i',j'}-y_{i'}(t_{j'};\mathbf{p}_{true})\})dx = y_{i}(t_{j};\mathbf{p}_{true}).
\end{align*}
Therefore,
\begin{align*}
\mathbb{E}_{\mathbf{p}_{n}^{G} \sim \pi_{p}^{G}}[y_{i}(t_{j};\mathbf{p}_{n}^{G})] 
&= \mathbb{E}_{\mathbf{p}_{n}^{G} \sim \pi_{p}^{G}}[\int x_{i,j}\pi_{e}^{G}(\{x_{i',j'}-y_{i'}(t_{j'};\mathbf{p}_{n}^{G})\})dx] \\
&= \int x_{i,j}\pi_{e}^{o}(\{x_{i',j'}-y_{i'}(t_{j'};\mathbf{p}_{true})\})dx = y_{i}(t_{j};\mathbf{p}_{true})
\end{align*} 

and 

\begin{align*}
\mathbb{E}_{\mathbf{p}_{n}^{G} \sim \pi_{p}^{G}}&[(y_{i}(t_{j};\mathbf{p}_{n}^{G})-y_{i}(t_{j};\mathbf{p}_{true}))^{2}]\\
&= \mathbb{E}_{\mathbf{p}_{n}^{G} \sim \pi_{p}^{G}}[y_{i}(t_{j};\mathbf{p}_{n}^{G})^{2}]-y_{i}(t_{j};\mathbf{p}_{true})^{2} \\
&= \mathbb{E}_{\mathbf{p}_{n}^{G} \sim \pi_{p}^{G}}[\int x_{i,j}^{2}\pi_{e}^{G}(\{x_{i',j'}-y_{i'}(t_{j'};\mathbf{p}_{n}^{G})\})dx - \text{Var}(e_{n,i}^{G}(t_{j}))]-y_{i}(t_{j};\mathbf{p}_{true})^{2} \\ 
&= (\int x_{i,j}^{2}\pi_{e}^{o}(\{x_{i',j'}-y_{i'}(t_{j'};\mathbf{p}_{true})\})dx - y_{i}(t_{j};\mathbf{p}_{true})^{2}) - \text{Var}(e_{n,i}^{G}(t_{j})) \\ 
&= \text{Var}(e_{n,i}^{o}(t_{j})) - \text{Var}(e_{n,i}^{G}(t_{j}))
\end{align*}

for each $i,j$.
\end{proof}

Notice that \cref{gan2} assumes that the mean of the generated data noise $\mathbf{e}^{G}$ is zero. \cref{gan2} implies that if the Wasserstein distance $d(Y^{o},Y^{G})$ between inferred and true trajectories successfully converges to zero, then the variance of the generated data noise $\mathbf{e}^{G}$ must not exceed the variance of the true data noise $\mathbf{e}^{o}$. Furthermore, if the variance of the generated data noise $\mathbf{e}^{G}$ can be simultaneously controlled while the Wasserstein loss is driven to zero, the following corollary can be immediately established, asserting that SiGMoID can accurately approximate the true solutions.

\begin{corollary}\label{gan3} Assuming that all conditions in \cref{gan2} are satisfied, if we additionally impose the condition $\displaystyle \text{Var}(e_{n,i}^{G}(t_{j}))=\text{Var}(e_{n,i}^{o}(t_{j}))$ for all $i \in I$, $j \in J$, then we have that \[\displaystyle y_{i}(t_{j};\mathbf{p}_{n}^{G}) = y_{i}(t_{j};\mathbf{p}_{true})\] for all $i \in I$, $j \in J$.
\end{corollary}

\subsection{Identifiability of dynamic systems}
Identifiability concerns whether a unique set of parameters can be recovered from the observed data of a dynamical system. Intuitively, if multiple parameter sets can generate exactly the same data, then the underlying system cannot be uniquely determined from those data. Therefore, identifiability analysis ensures that fitting a unique dynamical system to the given data is possible. Using the following definition and method, we demonstrate that the four dynamical systems considered in this manuscript (FitzHugh–Nagumo model, protein transduction model, \textit{Hes1} model, Lorenz system) are non-identifiable for the given  NSMC data.

\begin{definition} (Structural non-identifiability)
Let let $\mathbf{y}(t)=\mathbf{y}(t,\mathbf{p})$ be the model output (observable) in \cref{general_ode} as a function of parameters $\mathbf{p}\in\mathbb{R}^{d_p}$. A parameter $\mathbf{p}$ is said to be \emph{structurally non-identifiable} if there exists a parameter $\mathbf{p}' \ne \hat{\mathbf{p}}$ such that $$\mathbf{y}(t, \mathbf{p}') = \mathbf{y}(t, \hat{\mathbf{p}}) \quad \text{for all } t \in [0, T],$$

where $\hat{\mathbf{p}}$ denotes the (unknown) true parameter. Equivalently, the model output is invariant on the set $\mathcal{I} := \left\{ \mathbf{p} \in \mathbb{R}^{d_p} \;\middle|\; \mathbf{y}(t, \mathbf{p}) = \mathbf{y}(t, \hat{\mathbf{p}}) \;\; \forall t \in [0,T] \right\}$, indicating that $\mathbf{p}$ cannot be uniquely identified from noise-free observations of $\mathbf{y}(t)$.

\end{definition}

While structural non-identifiability can be analyzed rigorously, such approaches are typically restricted to relatively small or simplified (linear with respect to parameters) models due to their high computational complexity (see also \cite{raue2009structural}). One widely used approach is \emph{profile likelihood}, which evaluates the uncertainty of each parameter $p_i$ by fixing its value and optimizing the remaining parameters $\mathbf{p}_{-i}=\{p_1, \dots, p_{i-1}, p_{i+1}, \dots, p_{d_p}\}$ with respect to the objective function $L$ (e.g., the negative log-likelihood or residual sum of squares). Then, the profile likelihood for $p_i$ is defined as
$$\mathcal{L}_{\text{PL}}(p_i) := \min_{\mathbf{p}_{-i}} L(p_i, \mathbf{p}_{-i}).$$
Under this profile likelihood, we define the practical non-identifiability as below: 

\begin{definition} (Practical non-identifiability) A parameter estimate $\hat{p}_i$ is said to be \emph{practically non-identifiable} if the corresponding profile likelihood $\mathcal{L}_{\text{PL}}(p_i)$ remains approximately constant in one or both directions away from $\hat{p}_i$.
\end{definition}
This means the confidence interval, defined in \cref{CI}, is unbounded:

\begin{equation}\label{CI}
    \mathcal{I}_{\alpha}(p_i) := \left\{ p_i \in \mathbb{R} \;\middle|\; \mathcal{L}_{\text{PL}}(p_i) - \min_{\boldsymbol{p} \in \mathbb{R}^{d_p}} L(\boldsymbol{p}) \leq \Delta_\alpha \right\}
\end{equation}
where $\Delta_\alpha$ denotes the confidence threshold corresponding to the desired significance level $\alpha$. For example, $\Delta_{0.95} = 3.84$ for a 95\% confidence interval with one degree of freedom, following the chi-squared distribution. In this section, we defined the objective function $L$ as follows:
$$ L(\mathbf{p})=\frac{1}{|S_o|}\frac{1}{N_o}\sum_{i\in S_o}\sum_{j=1}^{N_o}|y_i^{o}(t_j)-y_i(t_j,\mathbf{p})|^2,$$
where $\{t_j\}$ denotes $N_o$ observation time points, and $S_o$ denotes the set of observable state indices. We compute the profile likelihoods, $\mathcal{L}_{\text{PL}}(p_i)$, using SciPy's \textit{solve\_ivp} to evaluate the model solution $\mathbf{y}$, and \textit{minimize} to derive $\mathcal{L}_{\text{PL}}(p_i
)$ from $ L(\mathbf{p})$. The resulting values are shown in \cref{fig:FN_CI}–\cref{fig:Lorenz_CI}, under the same experimental conditions described in the Experiments section:

\begin{figure}[h]

\vskip 0.1in
\begin{center}
\centerline{\includegraphics[width=0.9\columnwidth]{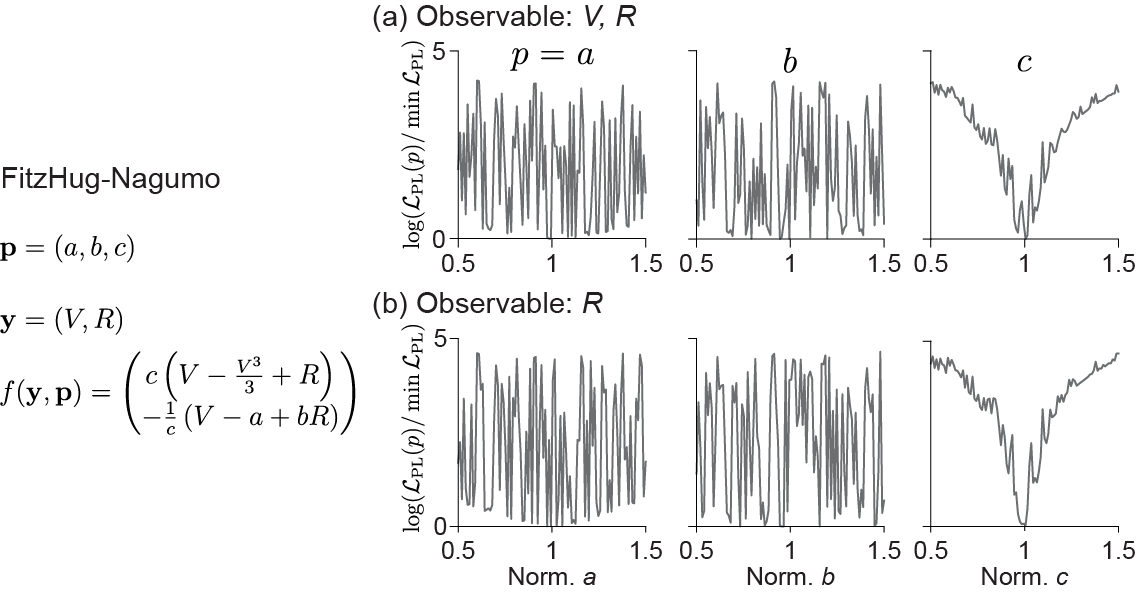}}
\caption{\textbf{Profile likelihoods $\mathcal{L}_{\textit{PL}}$ for FitzHug-Nagumo model.} Each panel displays values of $\mathcal{L}_{\textit{PL}}$ for a single parameter: $a$ (left), $b$ (middle), and $c$ (right) in FitzHug-Nagumo model (a-b). When two variables $V$ and $R$ are observable, $\mathcal{L}_{\textit{PL}}$ exhibits oscillatory behavior, although it still attains its minimum at the true parameter value (a, left). The parameter range on the x-axis is normalized by dividing by the true parameter value (centered at $a^{true}=0.3$, and the y-axis shows the scaled profile likelihood, obtained by dividing by $\min{}\mathcal{L}_{\textit{PL}}=\min_{a\in[0.5\times a^{true},1.5\times a^{true}]}{\mathcal{L}_{\textit{PL}}(a)}$, followed by taking the natural logarithm.
For parameter $b$, the global minimum cannot be determined because the minimum value of $\mathcal{L}_{\textit{PL}}$ is not uniquely attained (a, middle), indicating a non-identifiability issue. In contrast, parameter $c$ exhibit a clear global minimum at the center, implying that c is identifiable. When only $V$ is observable, the identifiability patterns of all three parameters remain similar to those obtained when both $V$ and $R$ are observed (b. left, middle, and right), indicating that observing $R$ does not resolve the identifiability issues.}
\label{fig:FN_CI}
\end{center}
\end{figure}

\begin{figure}[h!]

\vskip 0.1in
\begin{center}
\centerline{\includegraphics[width=\columnwidth]{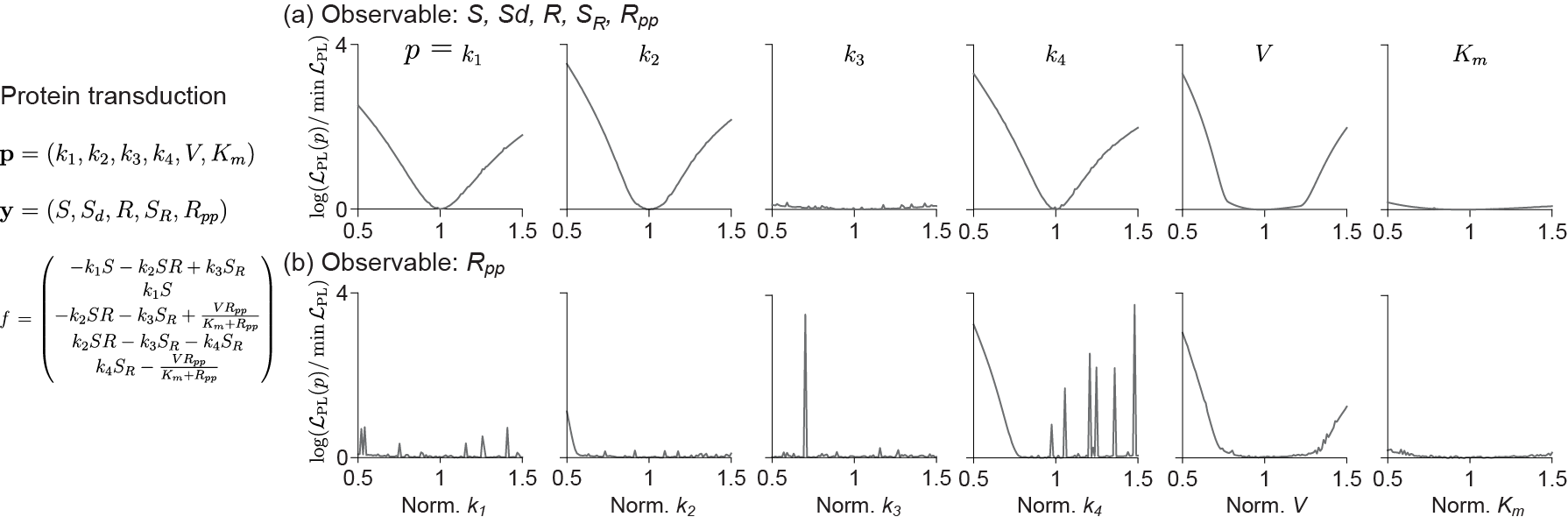}}
\caption{\textbf{Profile likelihoods for protein transduction model.} Each panel displays values of $\mathcal{L}_{\textit{PL}}$ for a single parameter $k_{1}$, $k_{2}$, $k_{3}$, $k_{4}$, $V$ and $K_{m}$ (from left to right) in protein transduction model (a-b). (a) When all variables are observable, the scaled values of $\mathcal{L}_{\textit{PL}}$ for $k_{1},k_{2},k_{4}$ and $V$ exhibit a clear global minimum at the center, whereas those for $k_{3}$ and $K_{m}$ remain nearly flat, indicating that these parameters are non-identifiable. (b) When only $R_{pp}$ is observed, all parameters except $V$  fail to exhibit a unique global minimum, indicating non-identifiability of all parameters.}
\label{fig:Protein_CI}
\end{center}
\end{figure}

\begin{figure}[h!]

\vskip 0.1in
\begin{center}
\centerline{\includegraphics[width=.8\columnwidth]{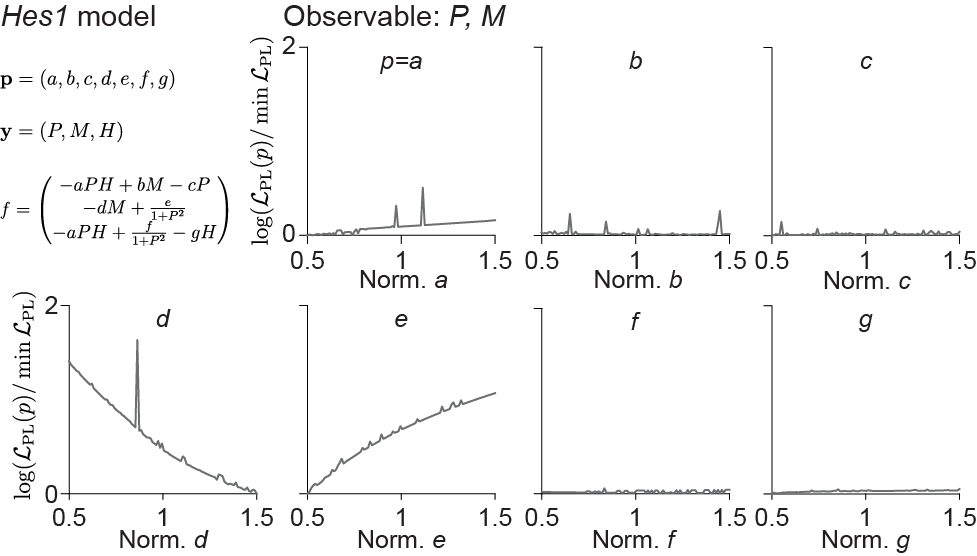}}
\caption{\textbf{Profile likelihoods for \textit{Hes1} model.} Each panel displays values of $\mathcal{L}_{\textit{PL}}$ for a single parameter $a$ to $g$ (from the left top to the right bottom) in \textit{Hes1} model. The scaled values of $\mathcal{L}_{\textit{PL}}$ for $b$, $c$, $f$ and $g$ remain nearly flat, indicating non-identifiability (right side). For parameters $a$, $d$, and $e$, the profile likelihood attains its minimum at the boundary of the parameter range rather than at the true value (center of the x-axis). Specifically, due to the presence of noise, the global minimum for parameters $a$, $d$, and $e$ is shifted, reflecting biased or distorted information about these parameters. Therefore, all parameters exhibit practical non-identifiability or indicate that the estimation problem is ill-posed within the restricted parameter domain.}
\label{fig:Hes1_CI}
\end{center}
\end{figure}

\begin{figure}[h!]

\vskip 0.1in
\begin{center}
\centerline{\includegraphics[width=\columnwidth]{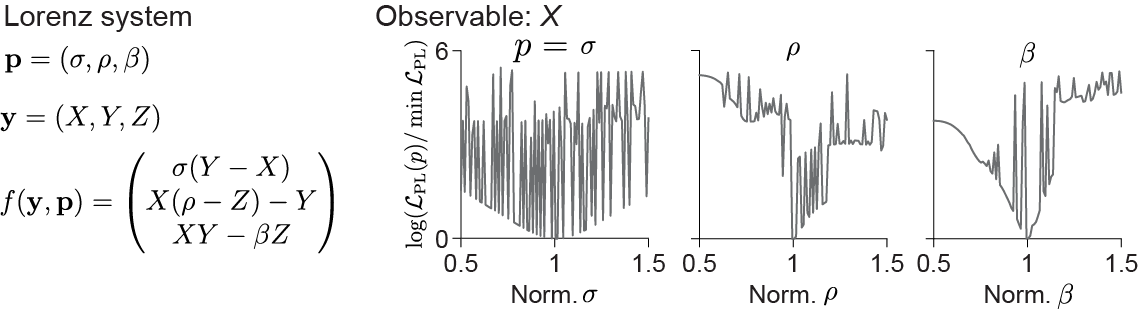}}
\caption{\textbf{Profile likelihoods for Lorenz system.} Each panel displays values of $\mathcal{L}_{\textit{PL}}$ for a single parameter $\sigma$ (left), $\rho$ (middle), and $\beta$ (right) in Lorenz system. All parameters exhibit a clear global minimum in the scaled values of $\mathcal{L}_{\textit{PL}}$, indicating that a single observation of $X$ is sufficient to ensure identifiability. However, noticeable fluctuations and the presence of relatively flat local minima ($b,c$) away from the true value also suggest limited practical identifiability.}
\label{fig:Lorenz_CI}
\end{center}
\end{figure}
\clearpage
\newpage
\section{Supplementary Tables}\label{Supp_tables}
For clarity and brevity, the following abbreviations are used throughout the Supplementary Tables: FitzHug-Nagumo (FN), Protein Transduction (P. TRANS.), \textit{Hes1} model (\textit{Hes1}), and Lorenz system (LORENZ). Note that all network units have fully-connected layers. 

\begin{table}[ht!]
\caption{Network settings and simulation data settings used for training HyperPINN.}
\label{table:network settings:hyperpinn}
\vskip 0.15in
\begin{center}
\begin{small}
\begin{sc}
\begin{tabularx}{\textwidth}{l *{4}{>{\centering\arraybackslash}X}}
\toprule
Setting & FN & P. Trans. & \textit{Hes1}  & Lorenz \\
\midrule
optimizer & \multicolumn{4}{c}{Adam} \\
\midrule
learning rate & $5\times10^{-4}$ & $10^{-5}$ & $10^{-4}$ & $10^{-4}$ \\
\midrule
batch size & \multicolumn{4}{c}{$10^{4}$} \\
\midrule
$N_{p}$ & $10^{3}$ & $10^{3}$ & $5\times10^{3}$ & $2\times10^{3}$ \\
\midrule
$\alpha$, $\beta$ & $1$, $0.001$ & $1$, $0$ & $1$, $0$ & $1$, $0$ \\
\midrule
training epochs & \multicolumn{4}{c}{$3\times 10^{4}$} \\
\midrule
Last activation function & sin & tanh & sin & elu \\
\bottomrule
\end{tabularx}
\end{sc}
\end{small}
\end{center}
\end{table}

\begin{table}[h]
\caption{Network settings used for training both the generator and discriminator of W-GAN.}
\label{table:network settings:wgan} \vskip 0.15in \begin{center}
\begin{small}
\begin{sc}
\begin{tabularx}{\textwidth}{l *{4}{>{\centering\arraybackslash}X}} \toprule Setting & FN & P. Trans. & \textit{Hes1} & Lorenz \\
\midrule optimizer & \multicolumn{4}{c}{Adam($\beta_{1}=0.0$, $\beta_{2}=0.9$)}\\ \midrule learning rate & $10^{-5}$ & $10^{-5}$& $5 \times 10^{-5}$ & $10^{-5}$\\ \midrule batch size & \multicolumn{4}{c}{full-batch}\\ \midrule $\lambda_{e}$ & $10^{2}$ &  $10^{4}$ &  $1$ &  $10^{2}$ \\\midrule noise dimension & \multicolumn{4}{c}{$32$} \\ \midrule training epochs & \multicolumn{4}{c}{$10^{5}$} \\ 
\bottomrule\end{tabularx} \end{sc} \end{small} \end{center} 
\end{table}

\begin{table}[ht!]
\caption{Parameter estimation results for the FN system obtained from four methods: SiGMoID, MAGI, FGPGM, and AGM. Numerical values in parentheses denote the true parameter values used to generate the NS and NSMC datasets. All methods report parameter estimates in the form of mean ± standard deviation; for SiGMoID, these statistics are computed over three repeated runs. In the NS setting, the method whose estimates most closely match the true parameter values is highlighted in \textbf{bold}.}
\label{table:params_FN}
\vskip 0.15in
\begin{center}
\begin{small}
\begin{sc}
\begin{tabularx}{\textwidth}{l *{4}{>{\centering\arraybackslash}X}}
\toprule
Data & Method & \makecell{$a$ \\ ($0.20$)} & \makecell{$b$ \\ ($0.20$)} & \makecell{$c$ \\ ($3.00$)} \\
\midrule
\multirow{4.2}{*}{NS}  & SiGMoID    & $\mathbf{0.20 \pm 0.00}$ & $\mathbf{0.19+0.01}$ & $\mathbf{2.99+0.00}$ \\
& MAGI & $0.19 \pm 0.02$ & $0.35 \pm 0.09$ & $2.89 \pm 0.06$ \\
& FGPGM    & $0.22 \pm 0.04$ & $0.32 \pm 0.13$ & $2.85 \pm 0.15$ \\
& AGM    & $0.30 \pm 0.03$ & $0.36 \pm 0.06$ & $2.04 \pm 0.14$ \\
\midrule
NSMC & SiGMoID    & $0.20\pm 0.00$ & $0.21+0.00$ & $2.99+0.00$  \\
\bottomrule
\end{tabularx}
\end{sc}
\end{small}
\end{center}
\end{table}

\begin{table}[ht!]
\caption{Parameter estimation results for the protein transduction system, comparing the mean values generated by SiGMoID for the NS and NSMC datasets.}
\label{table:params_pt}
\vskip 0.15in
\begin{center}
\begin{small}
\begin{sc}
\begin{tabularx}{\textwidth}{l *{7}{>{\centering\arraybackslash}X}}
\toprule
Data & Method & \makecell{$k_{1}$\\$(0.070)$} & \makecell{$k_{2}$\\$(0.600)$} & \makecell{$k_{3}$\\$(0.050)$} & \makecell{$k_{4}$\\$(0.300)$} & \makecell{$V$\\$(0.017)$} & \makecell{$K_{m}$\\$(0.300)$} \\
\midrule
NS & SiGMoID & $0.070\pm0.000$ & $0.604\pm0.001$ & $0.052\pm0.001$ & $0.300\pm0.000$ & $0.017\pm0.000$ & $0.315\pm0.001$ \\
\midrule
NSMC & SiGMoID & $0.088\pm0.001$ & $0.607\pm0.001$ & $0.011\pm0.002$ & $0.289\pm0.000$ & $0.015\pm0.000$ & $0.255\pm0.003$\\
\bottomrule
\end{tabularx}
\end{sc}
\end{small}
\end{center}
\end{table}

\begin{table}[ht!]
\caption{Parameter estimation results in the \textit{Hes1} system, comparing the mean values of generation from each four methods. Notably, the temporal dynamics of the unobservable component ($H$) are only determined by $P$, $H$, and three parameters $a, f$, and $g$ (the third equation in Equation (3)). In contrast, the other four parameters ($b, c, d$, and $e$) determine the true trajectories of $P$ and $M$, but various combinations of these four parameters can also reproduce the true $P$ and $M$, making their estimation non-unique. Therefore, accurately estimating $a, f$, and $g$ is sufficient to infer the true trajectory of $H$.}
\label{table:params_hes1} \vskip 0.15in \begin{center}
\begin{small}
\begin{sc}
\begin{tabularx}{\textwidth}{l *{7}{>{\centering\arraybackslash}X}} \toprule Method & $a$ $(0.022)$ & $b$ $(0.300)$ & $c$ $(0.031)$ & $d$ $(0.028)$ & $e$ $(0.500)$ & $f$ $(20.000)$ & $g$ $(0.300)$  \\ \midrule SiGMoID & \makecell{$0.031$\\$\pm0.001$} & \makecell{$0.310$\\$\pm0.007$} & \makecell{$0.381$\\$\pm0.000$} & \makecell{$0.025$\\$\pm0.000$}  & \makecell{$0.399$\\$\pm0.016$} & \makecell{$\mathbf{14.631}$\\$\mathbf{\pm0.638}$} & \makecell{$\mathbf{0.321}$\\$\mathbf{\pm0.014}$} \\ MAGI & \makecell{$\mathbf{0.021}$\\$\pm\mathbf{0.003}$}& \makecell{$0.329$\\$\pm0.051$} & \makecell{$0.035$\\$\pm0.006$} & \makecell{$0.029$\\$\pm0.002$}  & \makecell{$0.552$\\$\pm0.074$} & \makecell{$13.759$\\$\pm3.026$} & \makecell{$0.141$\\$\pm0.026$} \\ Ramsay et al. (2007) & \makecell{$0.027$\\$\pm0.026$} & \makecell{$\mathbf{0.302}$\\$\mathbf{\pm0.086}$} & \makecell{$\mathbf{0.031}$\\$\mathbf{\pm0.010}$} & \makecell{$\mathbf{0.028}$\\$\mathbf{\pm0.003}$}  & \makecell{$\mathbf{0.498}$\\$\pm0.088$} & \makecell{$604.9$\\$\pm5084.8$} & \makecell{$1.442$\\$\pm9.452$} \\ \bottomrule \end{tabularx} \end{sc} \end{small} \end{center} 
\end{table}

\begin{table}[ht!]
\caption{Parameter estimation results for the Lorenz system, presenting the mean values generated by SiGMoID.}
\label{table:params_lorenz} \vskip 0.15in \begin{center}
\begin{small}
\begin{sc}
\begin{tabularx}{\textwidth}{l *{3}{>{\centering\arraybackslash}X}}\toprule Method & \makecell{$\sigma$\\$(10.00)$} & \makecell{$\rho$\\$(28.00)$} & \makecell{$\beta$\\$(2.67)$} \\ \midrule SiGMoID & $10.60\pm0.003$ & $27.69\pm0.044$ & $2.65\pm0.035$ \\ \bottomrule \end{tabularx} \end{sc} \end{small} \end{center} 
\end{table}

\begin{table}[ht!]
\caption{\textbf{Computational cost.} Runtime results (mean $\pm$ standard deviation (sec)) for all experiments performed using SiGMoID. The GAN terminates training once the approximated Wasserstein distance reaches its minimum within 100,000 epochs. All neural network training is conducted on an NVIDIA RTX A5000 GPU.}
\label{table:runtime} \vskip 0.15in 
\begin{center}
\begin{small}
\begin{sc}
\begin{tabularx}{\textwidth}{l *{4}{>{\centering\arraybackslash}X}}
\toprule
Setting & FN & P. Trans. & \textit{Hes1}  & Lorenz \\
\midrule
Training time (sec) & $6757.4\pm42.6$ & $2525.4\pm23.9$ & $3658.9\pm45.2$ & $2521.2\pm6.4$ \\
\bottomrule
\end{tabularx}
\end{sc}
\end{small}
\end{center}
\end{table}

\end{document}